\documentclass[12pt]{my_class}

\usepackage{mathtools}

\usepackage{dsfont}

\usepackage{mathtools}
\DeclarePairedDelimiter\ceil{\lceil}{\rceil}
\DeclarePairedDelimiter\floor{\lfloor}{\rfloor}

\DeclareMathOperator*{\argmax}{argmax}

\newtheorem{assumption}[theorem]{Assumption}

\usepackage{booktabs}  

\usepackage{wrapfig}

\newcommand{\Greedy}{\textsc{Greedy}}

\title[Be Greedy in Multi-Armed Bandits]{Be Greedy in Multi-Armed Bandits}
\usepackage{times}

\altauthor{%
 \Name{Matthieu Jedor} \Email{matthieu.jedor@ens-paris-saclay.fr}\\
 \addr Centre Borelli, ENS Paris-Saclay \& Cdiscount
 \AND
 \Name{Jonathan Lou\"{e}dec} \Email{jonathan.louedec@cdiscount.com}\\
 \addr Cdiscount
 \AND
 \Name{Vianney Perchet} \Email{vianney.perchet@normalesup.org}\\
 \addr CREST, ENSAE Paris \& Criteo AI Lab%
}

\begin{document}

\maketitle

\begin{abstract}%
  The Greedy algorithm is the simplest heuristic in sequential decision problem that carelessly takes the locally optimal choice at each round, disregarding any advantages of exploring and/or information gathering. Theoretically, it is known to sometimes have poor performances, for instance even a linear regret (with respect to the time horizon) in the standard multi-armed bandit problem. On the other hand, this heuristic performs reasonably well in practice and it even has sublinear, and even near-optimal, regret bounds in some very specific  linear contextual and Bayesian bandit models.
  
  We build on a recent line of work and investigate bandit settings where the number of arms is relatively large and where simple greedy algorithms enjoy highly competitive performance, both in theory and in practice. We first provide a generic worst-case bound on the regret of the Greedy algorithm. When combined with some arms subsampling, we prove that it verifies near-optimal worst-case regret bounds in continuous, infinite and many-armed bandit problems. Moreover, for shorter time spans, the theoretical relative suboptimality of Greedy is even reduced.
  
  As a consequence, we subversively claim that for many interesting problems and associated horizons, the best compromise between theoretical guarantees, practical performances and computational burden is definitely to follow the greedy heuristic. We support our claim by many numerical experiments that show significant improvements compared to the state-of-the-art, even for moderately long time horizon. 
\end{abstract}


\section{Introduction}

Multi-armed bandits are basic instances of online learning problems with partial feedback \citep{bubeck2012regret, lattimore2018bandit, slivkins2019introduction}.
In the standard stochastic bandit problem, a learning agent sequentially pulls among a finite set of actions, or ``arms'', and observes a stochastic reward accordingly. 
The goal of the agent is then to maximize its cumulative reward, or equivalently, to minimize its regret, defined as the difference between the cumulative reward of an oracle (that knows the mean rewards of arms) and the one of the agent.
This problem requires to trade-off between exploitation (leveraging the information obtained so far) and exploration (gathering information on uncertain arms). 

The exploration, although detrimental in the short term, is usually needed in the worst-case as it ensures that the learning algorithm ``converges'' to the optimal arm in the long run. On the other hand, the \Greedy{} algorithm, an exploration-free strategy, focuses on pure exploitation and pulls the apparently best arm according to the information gathered thus far, at the risk of only sampling once the true optimal arm. This typically happens with Bernoulli rewards where only arms whose first reward is a 1 will be pulled again (and the others discarded forever). 
As a consequence, with some non-zero probability, the regret grows linearly with time as illustrated in the following example.

\begin{example}
\label{ex:failure}
Consider a relatively simple Bernoulli bandit problem consisting of $K=2$ arms with expected rewards $0.9$ and $0.1$ respectively. 
With probability at least 0.01, \Greedy{} fails to find the optimal arm. On the other hand, with probability $0.9^2$ it suffers no regret after the initial pulls. This results in a linear regret with a large variance.
This typical behavior is illustrated in Appendix \ref{sub:failure}.
\end{example}

Two solutions have been proposed to overcome this issue.
The first one is to force the exploration; for example with an initial round-robin exploration phase \citep{even2002pac}, or by spreading the exploration  uniformly over time à la \textsc{Epsilon-Greedy} \citep{auer2002finite}. However, both these algorithms need to know the different parameters of the problem to perform optimally (either to set the length of the round-robin phase or the value of $\varepsilon$), which represents a barrier to their use in practice.
The second solution is to have a data-driven and adaptive exploration;
for example, by adding an exploration term à la \textsc{UCB} \citep{auer2002finite}, by using a Bayesian update à la \textsc{Thompson Sampling} \citep{thompson1933likelihood,Perrault3}, by using data- and arm-dependent stopping times for exploring \`a la \textsc{Explore-Then-Commit} \citep{perchet2013, perchet2016}  or by tracking the number of pulls of suboptimal arms \citep{baransi2014sub, honda2010asymptotically, honda2015non}. With careful tuning, these algorithms are asymptotically optimal for specific reward distributions. Yet this asymptotic regime can occur after a long period of time \citep{garivier2019explore} and thus simpler heuristics might be preferable for relatively short time horizon \citep{vermorel2005multi, kuleshov2014algorithms}. 

Conversely, the simple \Greedy{} algorithm has recently been proved to satisfy near-optimal regret bounds in some linear contextual model \citep{bastani2017mostly, kannan2018smoothed, raghavan2020greedy} and a sublinear regret bound in some Bayesian many-armed setting \citep{bayati2020optimal}. In particular, this was possible because the \Greedy{} algorithm 
benefits from ``free'' exploration when the number of arms is large enough. 
We illustrate this behavior in the following example.
\begin{example}
Consider bandit problems where rewards are Gaussian distributions with unit variance and mean rewards are drawn i.i.d.\ from a uniform distribution over $[0, 1]$. In Figure \ref{fig:heatmap}, we compare the regret of \Greedy{} with the \textsc{UCB} algorithm for different number of arms and time horizon. 
For both algorithms, we observe a clear transition phase between problems with higher average regret (with darker colors) and problems with lower regret (with lighter colors). In this example, this transition takes the form of a diagonal.

This diagonal is much lower for \Greedy{} compared to \textsc{UCB}, meaning that \Greedy{} performs better in the problems in-between, and this in spite of \textsc{UCB} being optimal in the problem-dependent sense (on the other hand, that is when the horizon is large, \textsc{UCB} outperforms \Greedy{}).
The intuition is that, when the number of near-optimal arms is large enough, \Greedy{} rapidly converges to one of them while \textsc{UCB} is still in its initial exploration phase. The key argument here is the short time horizon relatively to the difficulty of the problem; we emphasis on the ``relatively'' as in practice the ``turning point'', that is the time horizon for which \textsc{UCB} performs better, can be extremely large.

\begin{figure}
    \includegraphics[width=0.99\linewidth]{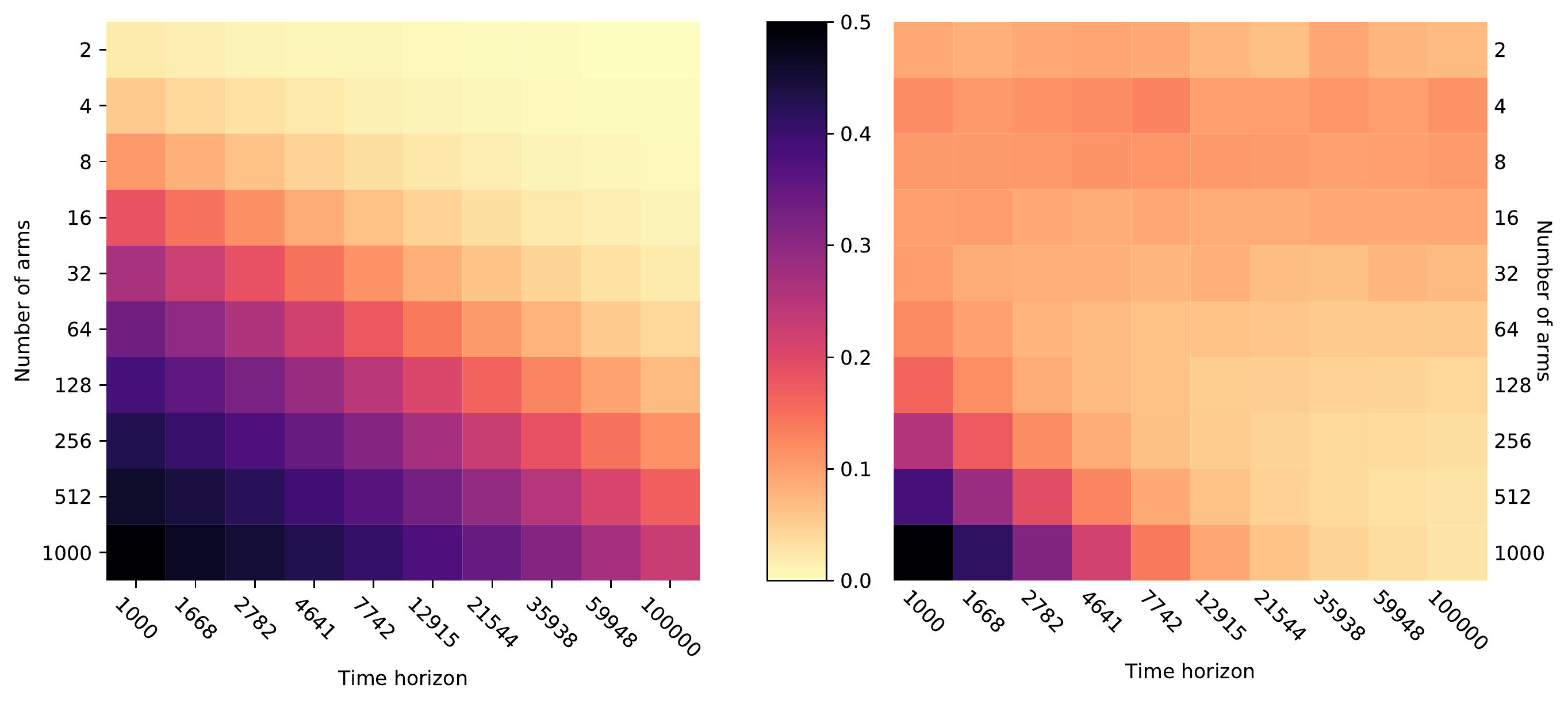}
    \caption{Bayesian regret divided by the horizon for \textsc{UCB} (left) and \Greedy{} (right) as a function of the number of arms and the horizon in Gaussian bandit problems. 
    }
    \label{fig:heatmap}
\end{figure}
\end{example}

Numerous interesting problems actually lie in the bottom left corner of Figure \ref{fig:heatmap}, i.e., bandit problems with a large number of arms and a relatively short time horizon and, as a consequence, the \Greedy{} algorithm should be considered as a valid baseline.

\paragraph{Our results}
We first provide a generic regret bound on \Greedy{}, and we illustrate how to derive worst-case regret bounds. We will then instantiate this regret bound to a uniformly sampled subset of arms and prove this satisfies near-optimal worst-case regret bounds in the continuous-armed, infinite-armed and many-armed bandit models. As a byproduct of our analysis, we get that the problem of unknown smoothness parameters can be overcome by a simple discretization depending only on the time horizon in the first of these models.
In all these settings, we repeat the experiments of previous papers and show that the \Greedy{} algorithm outmatches the state-of-the-art.

\paragraph{Detailed comparison with prior work on  \Greedy{}}
\Greedy{} recently regained some attention in Bayesian bandit problems with a large but finite number of arms \citep{bayati2020optimal}. It performs extremely well empirically when the number of arms is large, sometimes better than ``optimal'' algorithms; in that case, the regret of \Greedy{} is sublinear, though not optimal. In the following, we get rid of the strong Bayesian assumptions and we consider many different bandit models, where a subsampling technique is required and considered in the following.

Another recent success of \Greedy{} is in linear contextual bandit problems, as it is asymptotically optimal for a two-armed contextual bandit with linear rewards when a covariate diversity condition holds \citep{bastani2017mostly}. This idea can be extended to rewards given by generalized linear models.  
If observed contexts are selected by an adversary, but perturbed by white noise, then \Greedy{} can again have optimal regret guarantees \citep{kannan2018smoothed}. Additional assumptions can even improved those results
\citep{raghavan2018externalities, raghavan2020greedy}.
Those results hold because exploration is not needed thanks to the diversity in the contexts. We do not believe this assumption is satisfied in many practical scenarios and we are therefore rather interested in the implicit exploration of \Greedy{}. As a consequence, we shall no further consider the contextual framework (even if admittedly, our results could be generated via careful binning \citep{perchet2013}).
Interestingly, an extensive empirical study of contextual bandit algorithms found that \Greedy{} is actually the second most efficient algorithm and is extremely close to the first one \citep{bietti2018contextual}.

The \Greedy{} algorithm has already been shown to enjoy great empirical performance in the continuous-armed bandit model \citep{jedor2020lifelong}. In this paper, we make formal this insight.
Finally, we mention that in the one-dimensional linear bandit problem with a known prior distribution, the cumulative regret of a greedy algorithm (under additional structural assumptions) admits an $\mathcal{O}(\sqrt{T})$ upper bound and its Bayes risk admits an $\mathcal{O}(\log T)$ upper bound \citep{mersereau2009structured}. Linear bandits are only considered empirically in this paper (see Appendix \ref{app:exp_linear}).

\paragraph{Related work on bandit models}
We also provide a short literature review on the different bandit settings studied in this paper.
\subparagraph{Continuous-armed bandits}

In the continuous-armed bandit problem with nonparametric regularity assumptions \citep{agrawal1995continuum}, lower and upper bounds are matching up to sub-logarithmic factors \citep{kleinberg2005nearly}. Additional structural assumptions can be considered to lower regret, such as margin condition \citep{auer2007improved}, Lipschitz (w.r.t.\ some fixed metric) mean-payoff function \citep{kleinberg2008multi}, local Lipschitzness (w.r.t.\ some dissimilarity function) \citep{bubeck2010x}. Adaptivity to smoothness parameters is also a crucial task \citep{bubeck2011lipschitz, locatelli2018adaptivity, hadiji2019polynomial}.

\subparagraph{Infinite-armed bandits}

The original infinite-armed bandit problem \citep{berry1997bandit} consists in a sequence of $n$ choices from an infinite number of Bernoulli arms, with $n \rightarrow \infty$. The objective was to minimize the long-run failure rate. The Bernoulli parameters are independent observations from a known distribution. 
With a uniform prior distribution, it is possible to control the cumulative regret \citep{bonald2013two}. A more general model has been considered
\citep{wang2009algorithms}. 
In particular, rewards are usually assumed to be uniformly bounded in $[0, 1]$ and the mean reward of a randomly drawn arm is $\varepsilon$-optimal with probability $\mathcal{O}\left( \varepsilon^\beta \right)$ for some $\beta > 0$.

\subparagraph{Many-armed bandits}
Models in many-armed bandit problems are more varied, but the main idea is that the number of arms is large comparatively to the number of rounds \citep{teytaud2007anytime}. The exploration can be enhanced with a focus on a small subset of arms (using a cross-entropy based algorithm without theoretical guarantees thought) \citep{wang2017cemab}. 
The definition of regret can also be altered; by considering a given quantile fraction of the probability distribution over the mean rewards of arms \citep{chaudhuri2018quantile} or with respect to a ``satisfing'' action (the definition of a satisficing action is set by the learner) \citep{russo2018satisficing}.
Mean rewards can also be formulate with a semi-parametric model \citep{ou2019semi}.
A setting with multiple best/near-optimal arms without  any assumptions about the structure of the bandit instance has also been considered \citep{zhu2020multiple}. The objective there is to design algorithms that can automatically adapt to the unknown hardness of the problem.

\section{Preliminaries}

In the stochastic multi-armed bandit model, a learning agent interacts sequentially with a finite set of $K$ distributions $\mathcal{V}_1, \dots, \mathcal{V}_K$, called arms. At round $t \in \mathds{N}$, the agent chooses an arm $A_t$, which yields a stochastic reward $X_t$ drawn from the associated probability distribution $\mathcal{V}_{A_t}$. 
The objective is to design a sequential strategy maximizing the expected cumulative reward up to some time horizon~$T$. Let $\mu_1, \dots, \mu_K$ denote the mean rewards of arms, and $\mu^\star \coloneqq \max_{k \in [K]} \mu_k$ be the best mean reward. The goal is equivalent to minimizing the regret, defined as the difference between the expected reward accumulated by the oracle strategy always playing the best arm at each round, and the one accumulated by the strategy of the agent,
\begin{equation*}
    R_T = \mathbb{E} \left[ \sum_{t=1}^T \left( \mu^\star - X_t \right) \right] = T\mu^\star - \mathbb{E} \left[ \sum_{t=1}^T \mu_{A_t} \right]
\end{equation*}
where the expectation is taken with respect to the randomness in the sequence of successive rewards from each arm and the possible randomization in the strategy of the agent. Let $N_k(T)$ be the number of pulls of arm $k$ at the end of round $T$ and define the suboptimality gap of an arm $k  \in [K] \coloneqq \{ 1, \ldots, K \}$ as $\Delta_k = \mu^\star - \mu_k$. The expected regret is equivalently written as
\begin{equation*}
   R_T = \sum_{k=1}^K \Delta_k \mathbb{E} \left[ N_k(T) \right] \,.
\end{equation*}

\begin{wrapfigure}{R}{0.41\textwidth}
  \centering
  \vspace{-0.5cm}
    \begin{algorithm2e}[H]
    \DontPrintSemicolon
    \KwIn{Set of $K$ arms}
    
    \For{$t\gets 1$ \KwTo $T$}{
    Pull arm $\displaystyle A_t \in \argmax_{k \in [K]} \widehat{\mu}_k(t-1) $
    }
    
    \caption{\Greedy{}}
    \label{alg:greedy}
  \end{algorithm2e}
  \vspace{-0.6cm}
\end{wrapfigure}

\paragraph{The \Greedy{} algorithm} Summarized in Algorithm \ref{alg:greedy},  \Greedy{}  is probably the simplest and the most obvious algorithm. Given a set of $K$ arms, at each round $t$, it pulls the arm with the highest average reward $\displaystyle \widehat{\mu}_{k}(t-1) = \frac{1}{N_k(t-1)} \sum_{s=1}^{t-1} X_s \mathbf{1}\left\{ A_s = k\right\}$\footnote{With the convention that $0 / 0 = \infty$, so that the first $K$ pulls initialize each counter.}. Thus,  it constantly exploits the best empirical arm.

In the rest of the paper, 
we assume that the stochastic reward $X_t$ takes the form $X_t = \mu_{A_t} + \eta_t$ where $\{\eta_t\}_{t=1}^T$ are i.i.d.\ 1-subgaussian white noise and that $\mu_k$ are bounded for all $k \in [K]$, $\mu_k \in [0, 1]$ without loss of generality. We further assume the knowledge of the time horizon $T$, unknown time horizon can be handled as usual in bandit problems \citep{besson2018doubling}. Finally, we say that arm $k$ is $\varepsilon$-optimal for some $\varepsilon > 0$ if $\mu_k \geq \mu^\star - \varepsilon$.

\section{Generic bound on \Greedy{}}
\label{sec:generic}

We now present the generic worst-case regret bound on  \Greedy{}  that we will use to derive near-optimal bounds in several bandit models. The proof is provided in Appendix \ref{app:proof_generic_bound}.

\begin{theorem}
The regret of  \Greedy{}  verifies for all $\varepsilon > 0$
\begin{equation*}
R_T \leq T \exp\left( - N_\varepsilon \frac{\varepsilon^2}{2} \right) + 3 \varepsilon T + \frac{6 K}{\varepsilon} + \sum_{k=1}^K \Delta_k 
\end{equation*}
where $N_\varepsilon$ denotes the number of $\varepsilon$-optimal arms.
\label{thm:generic_bound_greedy}
\end{theorem}

\begin{remark}
This bound generalizes  a Bayesian analysis \citep{bayati2020optimal}. It is slightly looser; indeed the Bayesian assumption can be used to bound  $N_\varepsilon$ and further improve the third term by bounding the number of suboptimal arms. Those techniques  usually do not work in the stochastic setting.
\end{remark}

It is easy to see that this bound is meaningless when $N_\varepsilon$ is independent of $T$ as one of the first two terms will, at least, be linear with respect to $T$. 
On the other hand, $N_\varepsilon$ has no reason to depend on the time horizon.
The trick to obtain sublinear regret will be to lower bound $N_\varepsilon$ by a function of the number of arms $K$, then to optimize $K$ with respect to the time horizon $T$. To motivate this, consider the following example.

\begin{example}
Consider a problem with a huge number of arms $n$ with mean rewards drawn i.i.d.\ from a uniform distribution over $[0,1]$. In that specific case, we roughly have $N_\varepsilon \approx \varepsilon K$ for some subset of arms, chosen uniformly at random, with cardinality $K$. Taking $\varepsilon = \left( \frac{\log T}{K} \right)^{1/3}$, so that the first term in the generic bound is sublinear, yields a $\mathcal{O}\left( \max\left\{ T\left(\frac{\log T}{K}\right)^{1/3}, 
K \left(\frac{K}{\log T}\right)^{1/3} \right\} \right)$ regret bound, which comes from the second and third terms respectively. If we sub-sampled $K = T^{3/5} \left(\log T\right)^{2/5}$ arms, so that the maximum is minimized, the regret bound becomes $\mathcal{O} \left( T^{4/5} \left( \log T \right)^{1/5} \right)$; in particular it is sublinear. 
\end{example}

This argument motivates this paper and will be made formal in subsequent sections. Though this does not lead to optimal bounds -- as expected by the essence of the greedy heuristic in the multi-armed bandit model --, it will nonetheless be highly competitive for short time span in many practical bandit problems.

It is possible to theoretically improve the previous result by using a chaining/peeling type of argument. Unfortunately, it is not practical to derive better explicit guarantees as it involves an integral without close form expressions; its proof is postponed to Appendix \ref{CR:proof_chaining}.
\begin{corollary}\label{CR:chaining}
The regret of  \Greedy{}  verifies 
\begin{equation*}
R_T \leq \min_\varepsilon\Big\{ 3 \varepsilon T + \frac{6 K}{\varepsilon} + \int_{\varepsilon}^1 \left(3    T + \frac{6 K}{x^2}\right) \exp\left( - N_x \frac{x^2}{2} \right) dx\Big\} + T \exp\left( - \frac{K}{2} \right) + \sum_{k=1}^K \Delta_k \,.
\end{equation*}
\end{corollary}
\section{Continuous-armed bandits}

We first study \Greedy{} in the continuous-armed bandit problem. We recall that in this model, the number of actions is infinitely large. Formally, let $\mathcal{A}$ be an arbitrary set and $\mathcal{F}$ a set of functions from $\mathcal{A} \rightarrow \mathbb{R}$. The learner is given access to the action set $\mathcal{A}$ and function class $\mathcal{F}$. In each round $t$, the learner chooses an action $A_t \in \mathcal{A}$ and receives reward $X_t=f(A_t) + \eta_t$, where $\eta_t$ is some noise and $f \in \mathcal{F}$ is fixed, but unknown.
As usual in the literature \citep{kleinberg2005nearly, auer2007improved, hadiji2019polynomial}, we restrict ourselves to the case $\mathcal{A} = [0, 1]$, $\eta_t$ is 1-subgaussian, $f$ takes values in $[0, 1]$ and $\mathcal{F}$ is the set of all functions that satisfy an H\"older condition around the maxima. Formally,

\begin{assumption} \label{hyp:continuous}
There exist constants $L\geq0$ and $\alpha > 0$ such that for all $x\in[0,1]$,
\begin{equation*}
   f(x^\star) - f(x) \leq L \cdot |x^\star - x|^\alpha 
\end{equation*}
where $x^\star$ denotes the optimal arm.
\end{assumption}

This assumption captures the degree of continuity at the maxima and it is needed to ensure that this maxima is not reached at a sharp peak.

Similarly to CAB1 \citep{kleinberg2005nearly}, the \Greedy{} algorithm will work on a discretization of the action set into a finite set of $K$ equally spaced points $\{1/K, 2/K, \dots, 1\}$. Each point is then considered as an arm and we can apply the standard version of \Greedy{} on them. 

\begin{remark}
The same analysis holds if it chooses a point uniformly at random from the chosen interval $\left[\frac{k-1}{K}, \frac{k}{K} \right]$ for $1 \leq k \leq K$, see also \citet{auer2007improved}.
\end{remark}

The problem is thus to set the number of points $K$.
The first regret bound on the \Greedy{} algorithm assumes that the smoothness parameters are known. The proof is provided in Appendix \ref{app:proof_continuous}.

\begin{theorem}
\label{thm:continuous}
If $f:[0, 1] \rightarrow [0, 1]$ satisfies Assumption \ref{hyp:continuous}, then for all $\varepsilon > 0$ and a discretization of $K \geq \left( \frac{L}{\varepsilon}\right)^{1/\alpha}$ arms, the regret of \Greedy{} verifies 

\begin{equation*}
R_T \leq T \exp\left( - \frac{K}{2 L^{1/\alpha}} \varepsilon^{2 + 1/\alpha} \right) + 4 \varepsilon T + \frac{6 K}{\varepsilon} + K \,.
\end{equation*}

In particular, the choice

\begin{equation*}
    K = \left(32/27 \right)^{\alpha / (4\alpha+1)} L^{2/(4\alpha+1)} T^{(2\alpha+1)/(4\alpha+1)} \left( \log T \right)^{2\alpha/(4\alpha+1)}
\end{equation*}

yields for $L \leq \sqrt{\frac{3}{2T}} K^{\alpha + 1/2}$,

\begin{equation*}
R_T \leq 13 L^{2/(4\alpha+1)} T^{(3\alpha+1)/(4\alpha+1)} \left( \log T \right)^{2\alpha/(4\alpha+1)} + 1 \,.
\end{equation*}
\label{thm:bound_continuous}
\end{theorem}

This bound is sublinear with respect to the time horizon $T$, yet suboptimal. Indeed, the lower bound in this setting is $\Omega\left( T^{(\alpha+1)/(2\alpha+1)} \right)$ and the \textsc{MOSS} algorithm run on a optimal discretization attains it since its regret scales, up to constant factor, as $\mathcal{O}\left( L^{1/(2\alpha+1)} T^{(\alpha+1)/(2\alpha +1)} \right)$ \citep{hadiji2019polynomial}. 
Yet, as mentioned previously, \Greedy{} is theoretically competitive for short time horizon due to small constant factors. In Figure \ref{fig:bounds_known}, we displayed regret upper bounds of \textsc{MOSS} and \Greedy{} as a function of time for functions that satisfy Assumption \ref{hyp:continuous} with smoothness parameters $L=1$ and $\alpha=1$. We see that the bound on \Greedy{} is stronger up until a moderate time horizon $T \approx 12000$. 

Of course, assuming that the learner knows smoothness parameters $\alpha$ and $L$ is often unrealistic. If we want to ensure a low regret on very regular functions, by taking $\alpha \rightarrow \infty$, we have the following corollary.
\begin{corollary}
If $f:[0, 1] \rightarrow [0, 1]$ satisfies Assumption \ref{hyp:continuous}, then for a discretization of $K = \sqrt{\frac{4}{3} T \log T}$ arms, the regret of the \Greedy{} algorithm verifies for $L \leq 3^{1/4} \left( 4/3 \right)^{(2\alpha+1)/4} T^{2\alpha} \left( \log T \right)^{(\alpha+1)/2}$,

\begin{equation*}
R_T \leq 15 \max\{ L^{1/(2\alpha+1)}, L^{-1/(2\alpha+1)} \} T^{(3\alpha+2)/(4\alpha+2)} \sqrt{\log T} + 1 \,.
\end{equation*}
\label{cor:greedy_emp}
\end{corollary}

\begin{proof}
It is a direct consequence of Theorem \ref{thm:bound_continuous} with the choice of $\varepsilon = \left( L^{1/\alpha} \sqrt{\frac{3\log T}{T}} \right)^{\alpha/(2\alpha +1)}$.
\end{proof}

Once again, \Greedy{} attains a sublinear, yet suboptimal, regret bound.
In the case of unknown smoothness parameters, the regret lower bound is $\Omega\left( L^{1/(1+\alpha)} T^{(\alpha+2)/(2\alpha+2)} \right)$ \citep{locatelli2018adaptivity}, which is attained by
MeDZO with a $\mathcal{O}\left( L^{1/(\alpha+1)} T^{(\alpha+2)/(2\alpha+2)} \left( \log_2 T \right)^{3/2} \right)$ regret bound \citep{hadiji2019polynomial}. 
This time, \Greedy{} also has a lower polynomial dependency which makes it even more competitive theoretically. In Figure \ref{fig:bounds_unknown}, we displayed regret upper bounds of \textsc{MeDZO} and \Greedy{} (with unknown smoothness parameters) as a function of time for functions that satisfy Assumption \ref{hyp:continuous} with smoothness parameters $L=1$ and $\alpha=1$. Here we cannot see the turning point since \Greedy{} is stronger up until an extremely large time horizon $T \approx 1,9 \cdot 10^{46}$. 
Our numerical simulations will further support this theoretical advantage.

\begin{figure}
  \floatconts
  {fig:bounds_continuous}
  {\caption{Regret upper bound of various algorithms as a function of time in the continuous-armed bandit model with smothness parameters $L=1$ and $\alpha=1$.}}
  {%
    \subfigure[Known smoothness]{\label{fig:bounds_known}%
      \includegraphics[width=0.49\linewidth]{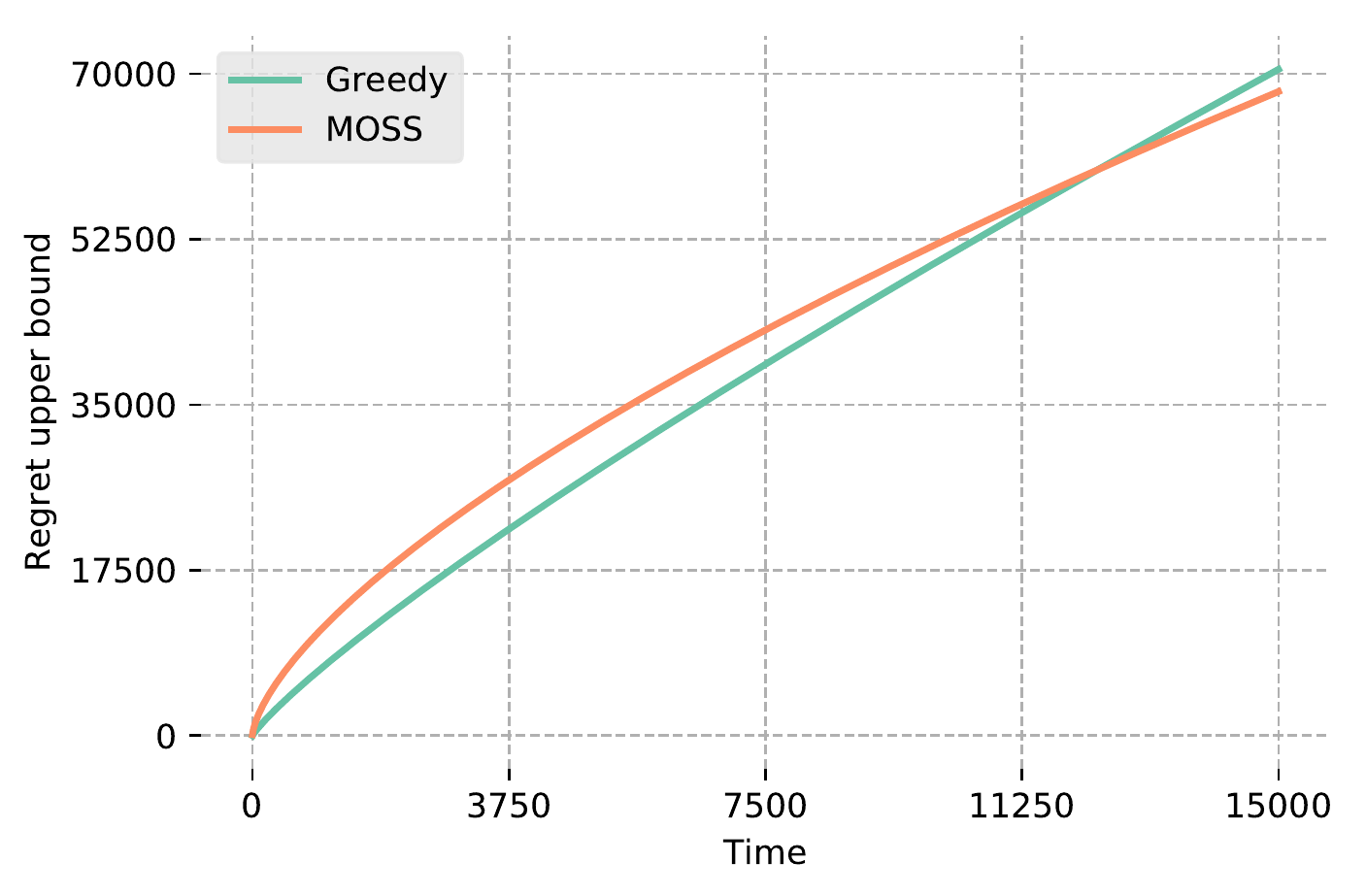}}%
    \subfigure[Unknown smoothness]{\label{fig:bounds_unknown}%
      \includegraphics[width=0.49\linewidth]{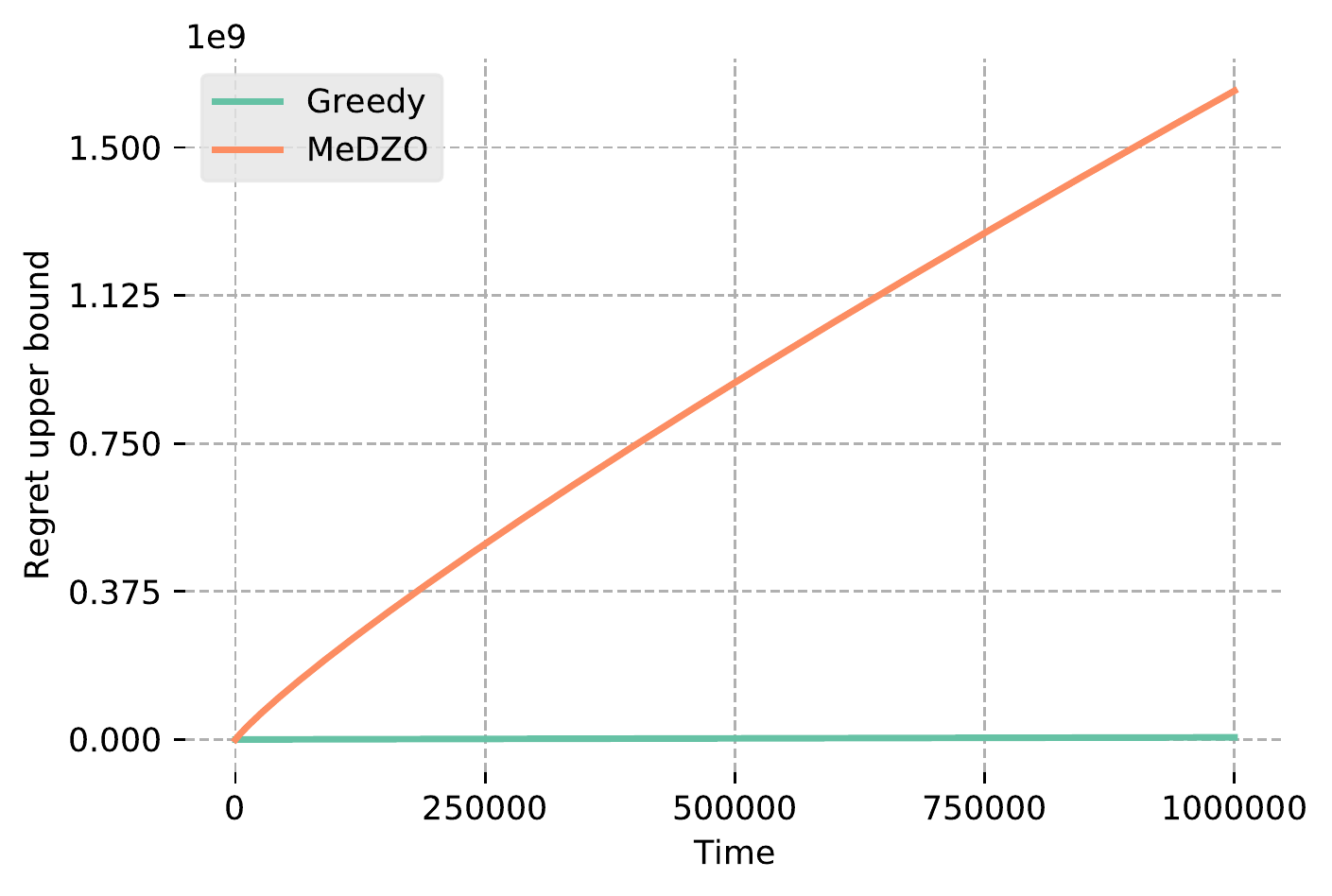}}
  }
\end{figure}

\section{Infinite-armed bandits}

We now study the infinite-armed bandit problem. In this setting, we consider the general model of \citet{wang2009algorithms}. In particular they assume a margin condition on the mean reward of a randomly drawn arm. Formally, 
\begin{assumption} \label{hyp:infinite}
There exist $\mu^\star \in (0,1]$ and $\beta >0$ such that the mean reward $\mu$ of a randomly drawn arm satisfies
    \begin{equation*}
 \mathbb{P}\left( \mu > \mu^\star - \varepsilon \right) = \mathcal{O}\left( \varepsilon^\beta \right) \text{, for } \varepsilon \rightarrow 0 \,.
    \end{equation*}
    Equivalently, there exist $c_1 > 0$ and $c_2 > 0$ such that
\begin{equation*}
    c_1 \varepsilon^\beta \leq \mathbb{P}\left( \mu > \mu^\star - \varepsilon \right) \leq c_2 \varepsilon^\beta \,.
\end{equation*}
\end{assumption}

Similarly to \textsc{UCB-F} \citep{wang2009algorithms}, \Greedy{} will consist of randomly choosing $K$ arms at first and then running \Greedy{} on those arms. The problem is then to choose the optimal number of arms $K$.
The following bound on \Greedy{} assumes the knowledge of the parameter $\beta$ and $c_1$. Its proof is deferred in Appendix \ref{app:proof_infinite}.

\begin{theorem}
\label{thm:infinite}
Assume Assumption \ref{hyp:infinite} of the model. The regret of  \Greedy{}  verifies for any subsampling of $K>0$ arms and for all $\varepsilon > 0$
\begin{equation*}
R_T \leq T \left[ \exp\left( - \frac{c_1}{4} K \varepsilon^{2 + \beta} \right) + \exp\left( - \frac{c_1}{8} K \varepsilon^{\beta} \right) \right] + 4 \varepsilon T + \frac{6 K}{\varepsilon} + K \,.
\end{equation*}
In particular, the choice
\begin{equation*}
    K = \left(2/3 \right)^{(2+\beta)/(4+\beta)} \left( \frac{8}{c_1(4+\beta)} \right)^{2/(4+\beta)} T^{(2+\beta)/(4+\beta)} \left( \log T \right)^{2/(4+\beta)} 
\end{equation*}
yields

\begin{equation*}
R_T\leq 20 \left(c_1(4+\beta)\right)^{-2/(4+\beta)} T^{(3+\beta)/(4+\beta)} \left( \log T \right)^{2/(4+\beta)} \,.
\end{equation*}
\label{thm:infinite_armed}
\end{theorem}

In comparison, the lower bound is this model is $\Omega\left( T^{\beta / (1+\beta)} \right)$ for any $\beta>0$ and $\mu^\star \leq 1$ and \textsc{UCB-F} obtained a $\mathcal{O}\left( T^{\beta/(\beta+1)} \log T \right)$ regret bound in the case $\mu^\star = 1$ or $\beta >1$ and a $\widetilde{\mathcal{O}}\left( T^{1/2} \right)$ bound otherwise \citep{wang2009algorithms}. The regret of \Greedy{} is once again sublinear, though suboptimal, with a lower logarithmic dependency. 
Our numerical simulations will further emphasis its competitive performance.

The case of unknown parameters is more complicated to handle compared to the continuous-armed model and is furthermore not the main focus of this paper. A solution proposed by \citet{carpentier2015simple} nonetheless, is to perform an initial phase to estimate the parameter $\beta$. 

\section{Many-armed bandits}

We now consider the particular model of many-armed bandit problem of \citet{zhu2020multiple}. It is somehow related to the previous two except it also takes into account the time horizon. In particular, it focuses on the case where multiple best arms are present. Formally, let $T$ be the time horizon, $n$ be the total number of arms and $m$ be the number of best arms. We emphasis that $n$ can be arbitrary large and $m$ is usually unknown. The following assumption will lower bound the number of best arms.

\begin{assumption} \label{hyp:many}
There exists $\gamma \in [0,1]$ such that the number of best arms satisfies
    \begin{equation*}
        \frac{n}{m} \leq T^\gamma \,.
    \end{equation*}
\end{assumption}

We assume that the value $\gamma$ (or at least some upper-bound) is known in our case, even though adaptivity to it is possible \citep{zhu2020multiple}. The following Theorem bounds the regret of a \Greedy{} algorithm that initially subsamples a set of arms. Its proof is provided in Appendix \ref{app:proof_many}.

\begin{theorem}
\label{thm:many}
Assume Assumption \ref{hyp:many} of the model and that the number of arms $n$ is large enough for the following subsampling schemes to be possible. 
Depending on the value of $\gamma$ and the time horizon $T$, it holds:
\begin{itemize}
    \item If $T^{1-3\gamma} \leq \log T$, in particular for $\gamma \geq \frac{1}{3}$ and $T \geq 2$, choosing $K = 2 T^{2\gamma}\log T$ leads to
    \begin{equation*}
        R_T \leq 14 T^{\gamma + 1/2} \log T + 2 \,.
    \end{equation*}
    \item Otherwise, the choice of $K=2\sqrt{T^{1+\gamma}\log T}$ yields
    \begin{equation*}
        R_T \leq 14 T^{(3+\gamma)/4} \sqrt{\log T} + 2 \,.
    \end{equation*}
\end{itemize}
\label{thm:many_armed}
\end{theorem}

The previous bounds indicate that \Greedy{} realizes a sublinear worst-case regret on the standard multi-armed bandit problem at the condition that the number of arms is large and the proportion of near-optimal arms is high enough.
To compare, the \textsc{MOSS} algorithm run on an optimal subsampling achieves a $\mathcal{O}\left(T^{(1+\gamma)/2}\log T\right)$ regret bound for all $\gamma \in [0, 1]$, which is optimal up to logarithmic factors \citep{zhu2020multiple}.
In this case, our numerical simulation will show that \Greedy{} is competitive even when the setup is close to the limit of the theoretical guarantee of \Greedy{}.

\section{Experiments}

We now evaluate \Greedy{} in the previously studied bandit models to highlight its practical competitive performance. For fairness reasons with respect to the other algorithms, and in the idea of reproducibility, we will not create new experiment setups but reproduce experiments that can be found in the literature (and compare the performances of \Greedy{} w.r.t.\ state of the art algorithms).

\subsection{Continuous-armed bandits}
\label{sec:exp_continuous}

In the continuous-armed bandit setting, we repeat the experiments of \citet{hadiji2019polynomial}. We consider three functions that are gradually sharper at the maxima and thus technically harder to optimize: 
\begin{align*}
    f_1 : x &\mapsto 0.5 \sin(13x)\sin(27x) + 0.5 \\
    f_2 : x &\mapsto\max\left( 3.6 x (1-x), 1 - | x - 0.05 | / 0.05 \right) \\
    f_3 : x &\mapsto x (1-x) \left( 4 - \sqrt{| \sin{60x} |} \right)
\end{align*}
These functions verify Assumption \ref{hyp:continuous} with $\alpha = 2, 1, 0.5$ and $L\approx221, 20, 2$, respectively, and are plotted for convenience in Appendix \ref{sub:continuous}.
Noises are drawn i.i.d.\ from a standard Gaussian distribution and we consider a time horizon $T=100000$. We compare the \Greedy{} algorithm with \textsc{MeDZO} \citep{hadiji2019polynomial}, \textsc{CAB1} \citep{kleinberg2005nearly} with \textsc{MOSS} \citep{audibert2009minimax,DegenneMoss} as the underlying algorithm and \textsc{Zooming} \citep{kleinberg2008multi}. 
For \Greedy{}, we use the discretization of Corollary \ref{cor:greedy_emp} while for \textsc{CAB.MOSS} we choose the optimal discretization $K=\ceil*{L^{2/(2\alpha+1)} T^{1/(2\alpha+1)}}$. For \textsc{MeDZO}, we choose the parameter suggested by authors $B = \sqrt{T}$. 
We emphasis here that \textsc{CAB.MOSS} and \textsc{Zooming} require the smoothness parameters contrary to \textsc{MeDZO} and \Greedy{}.
Results are averaged over $1000$ iterations and are presented on Figure \ref{fig:exp_continuous}. Shaded area represents 5 standard deviation for each algorithm.

\begin{figure}
  \floatconts
  {fig:exp_continuous}
  {\caption{Regret of various algorithms as a function of time in continuous-armed bandit problems.}}
  {%
    \subfigure[$f_1$]{
      \includegraphics[width=0.33\linewidth]{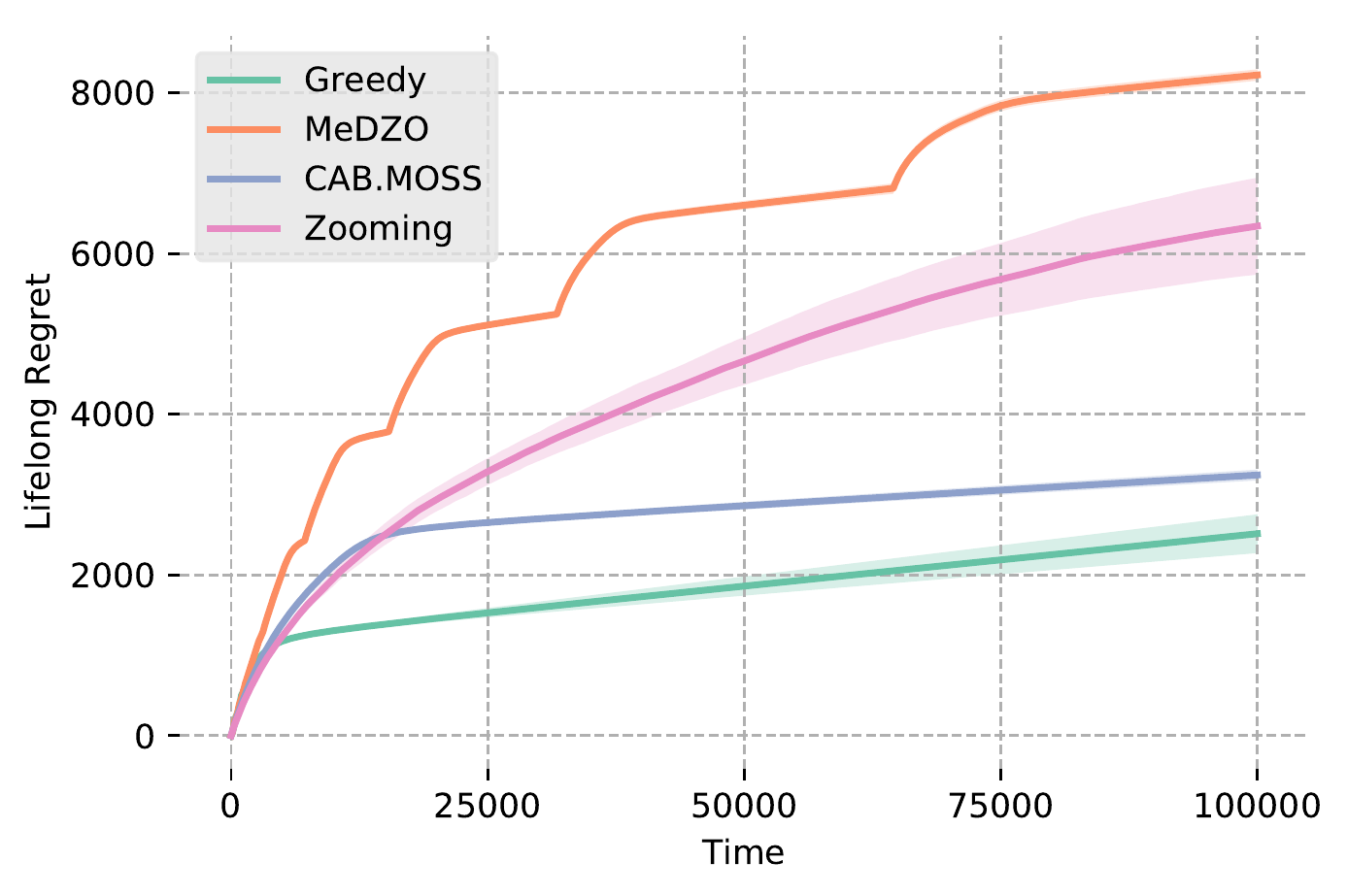}}%
    \subfigure[$f_2$]{
      \includegraphics[width=0.33\linewidth]{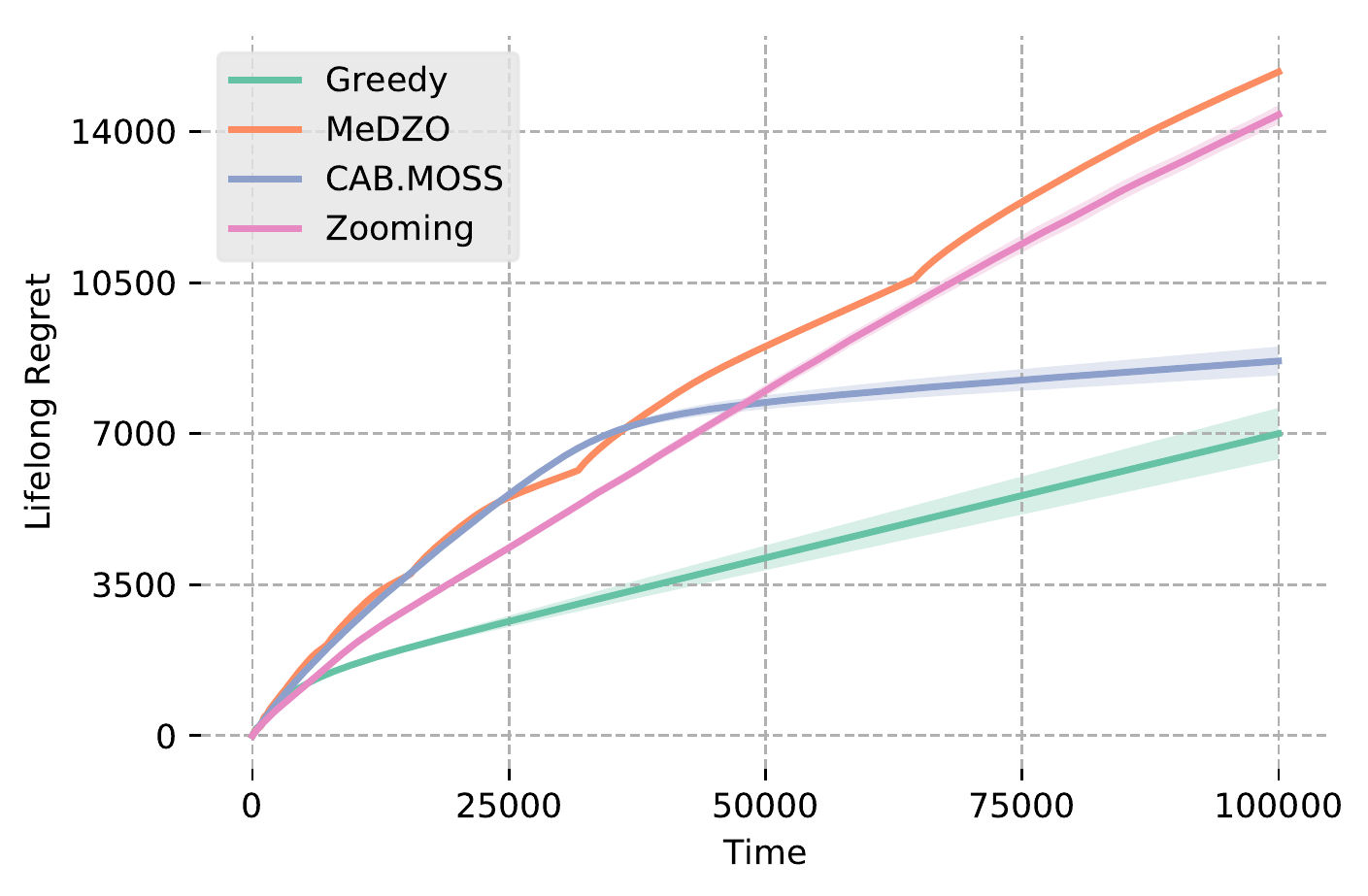}}
    \subfigure[$f_3$]{
      \includegraphics[width=0.33\linewidth]{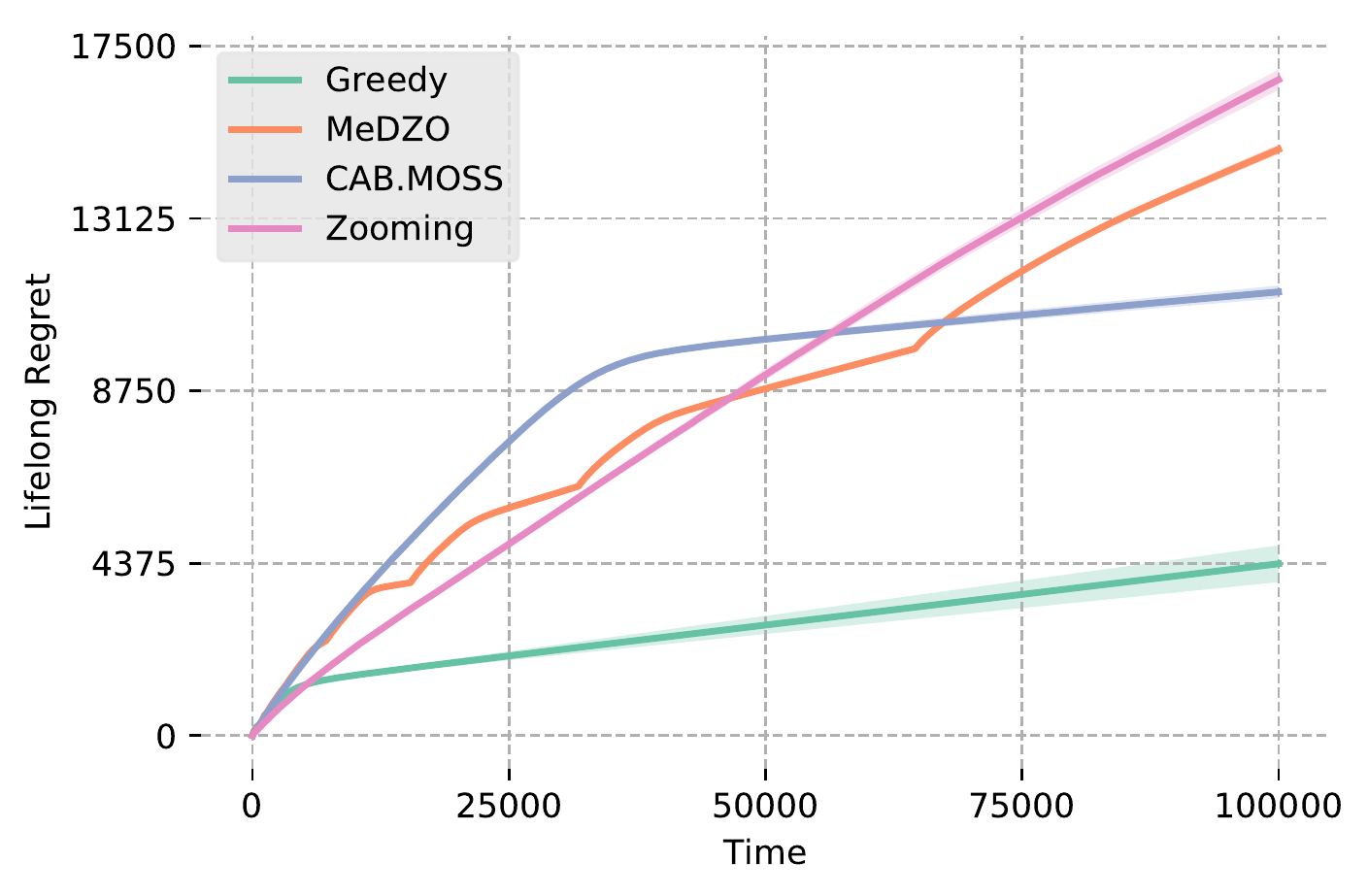}}
  }
\end{figure}

We see that \Greedy{} outperforms the other algorithms in all scenarios. We can clearly observe that the slope of the cumulative regret of \Greedy{} is stepper than the one of \textsc{CAB.MOSS}, yet it manages to obtain a lower regret by quickly concentrating on near-optimal arms. Moreover, the difference is striking for the relatively large time horizon considered here.  
Interestingly, the slope of \Greedy{} is more pronounced in the second scenario; this may be due to the low number of local maxima which negatively affects the number of $\varepsilon$-optimal arms for \Greedy{}.

\subsection{Infinite-armed bandits}

In the infinite-armed bandit setting, we repeat the experiments of \citet{bonald2013two}. We consider two Bernoulli bandit problems with a time horizon $T=10000$. In the first scenario, mean rewards are drawn i.i.d.\ from the uniform distribution over $[0, 1]$, while in the second scenario, they are drawn from a Beta(1, 2) distribution. We assume the knowledge of the parameters. We compare \Greedy{} with \textsc{UCB-F} \citep{wang2009algorithms}, a straightforward extension of \textsc{MeDZO} (analyzed by \citet{zhu2020multiple} in this model) and \textsc{TwoTarget} \citep{bonald2013two} that further assumes Bernoulli rewards and the knowledge of the underlying distribution of mean rewards. For \Greedy{}, we use the subsampling suggested in Theorem \ref{thm:infinite_armed}. Results, averaged over 1000 iterations, are displayed on Figure \ref{fig:exp_infinite} and the shaded area represents 0.5 standard deviation for each algorithm.

\begin{figure}
  \floatconts
  {fig:exp_infinite}
  {\caption{Regret of various algorithms as a function of time in infinite-armed bandit problems.}}
  {%
    \subfigure[Uniform prior]{
      \includegraphics[width=0.49\linewidth]{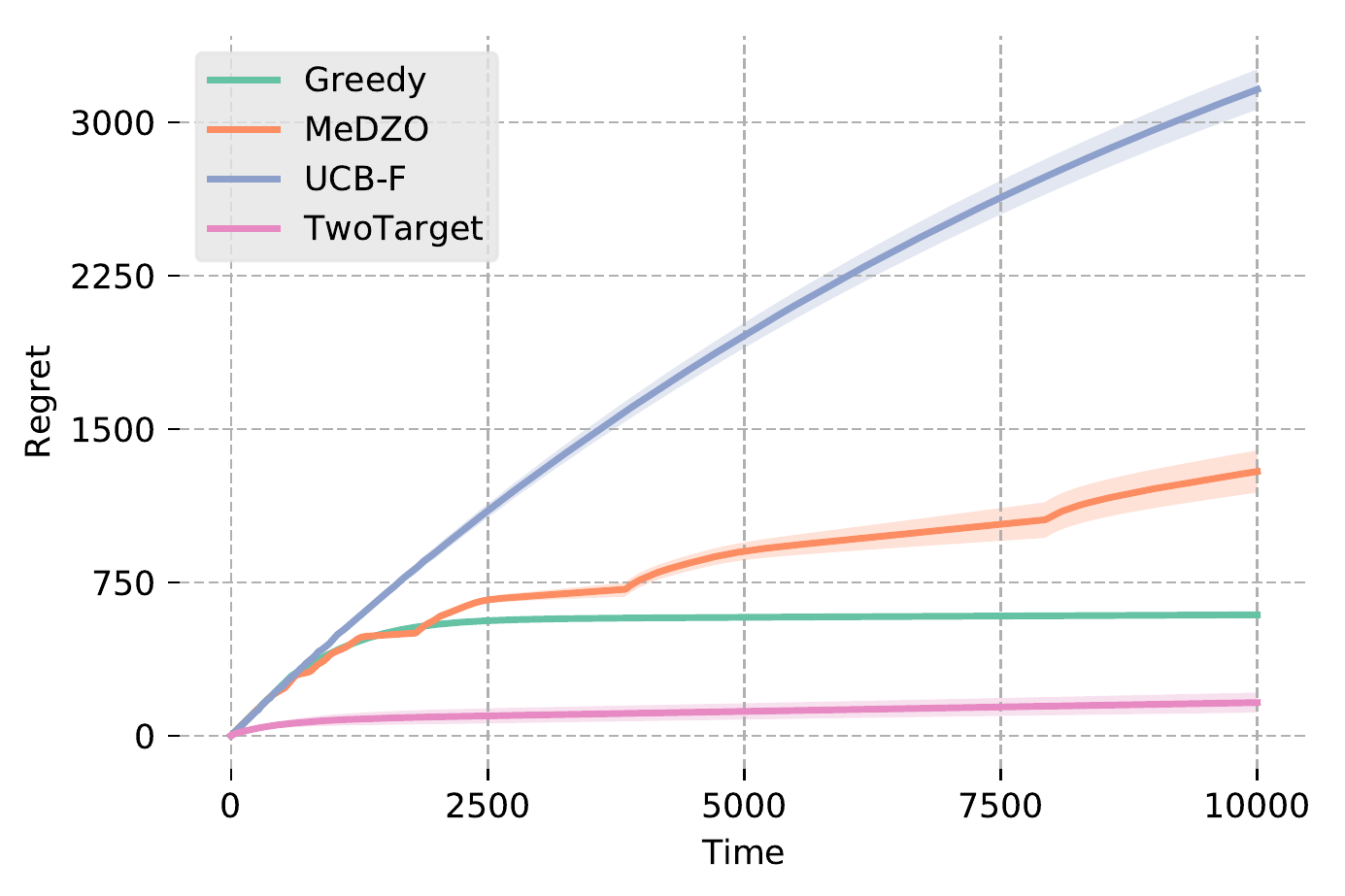}}%
    \subfigure[Beta(1, 2) prior]{
      \includegraphics[width=0.49\linewidth]{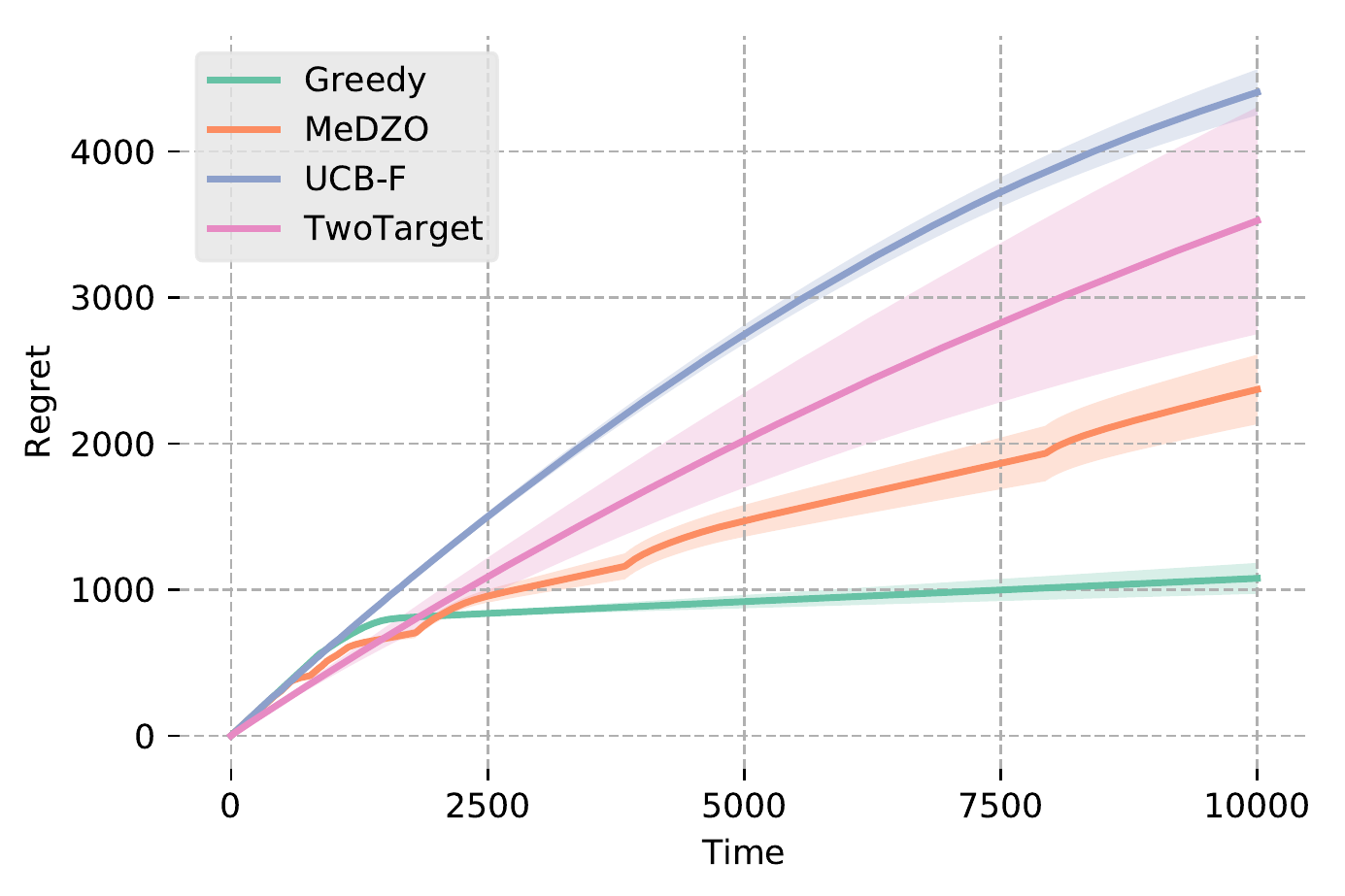}}
  }
\end{figure}

Once again, we see the excellent empirical performances of \Greedy{}. It is actually outperformed by \textsc{TwoTarget} in the uniform case since the latter has been specifically optimize for that case (and is asymptotically optimal) but \Greedy{} is more robust as the second scenario points out; furthermore, \textsc{TwoTarget} works only for Bernoulli rewards contrary to \Greedy{}.

\subsection{Many-armed bandits}

In the many-armed bandit setting, we repeat the experiment of \citet{zhu2020multiple}. 
We consider a Bernoulli bandit problem where best arms have an mean reward of  $0.9$ while for suboptimal arms they are evenly distributed among $\{0.1,0.2,0.3,0.4,0.5\}$. The time horizon is $T = 5000$ and the total number of arms $n = 2000$. We set the hardness level at $\gamma = 0.4$ resulting in a number of best arms $m = \ceil*{\frac{n}{T^\gamma}} = 64$. In this setup, \Greedy{} is near its limit in terms of theoretical guarantee.
We compare \textsc{OracleGreedy}, the greedy algorithm run on an subsampling of arms analyzed previously, with \textsc{MOSS} \citep{audibert2009minimax}, \textsc{OracleMOSS} \citep{zhu2020multiple} (which consider an optimal subsampling for \textsc{MOSS}), \textsc{MeDZO} \citep{hadiji2019polynomial, zhu2020multiple} and the standard \Greedy{} algorithm that consider all arms. For \textsc{OracleGreedy}, we consider a subsampling of $K=(1-2\gamma) T^{2\gamma} \log T / 4$ arms, which corresponds to the value of a more careful analysis of the regret in the bad events in Theorem \ref{thm:many_armed} for 1/4-subgaussian random variables.
Results are averaged over 5000 iterations and displayed on Figure \ref{fig:many_armed}. Shaded area represents 0.5 standard deviation for each algorithm.

\begin{figure}
    \centering
    \includegraphics[width=0.5\linewidth]{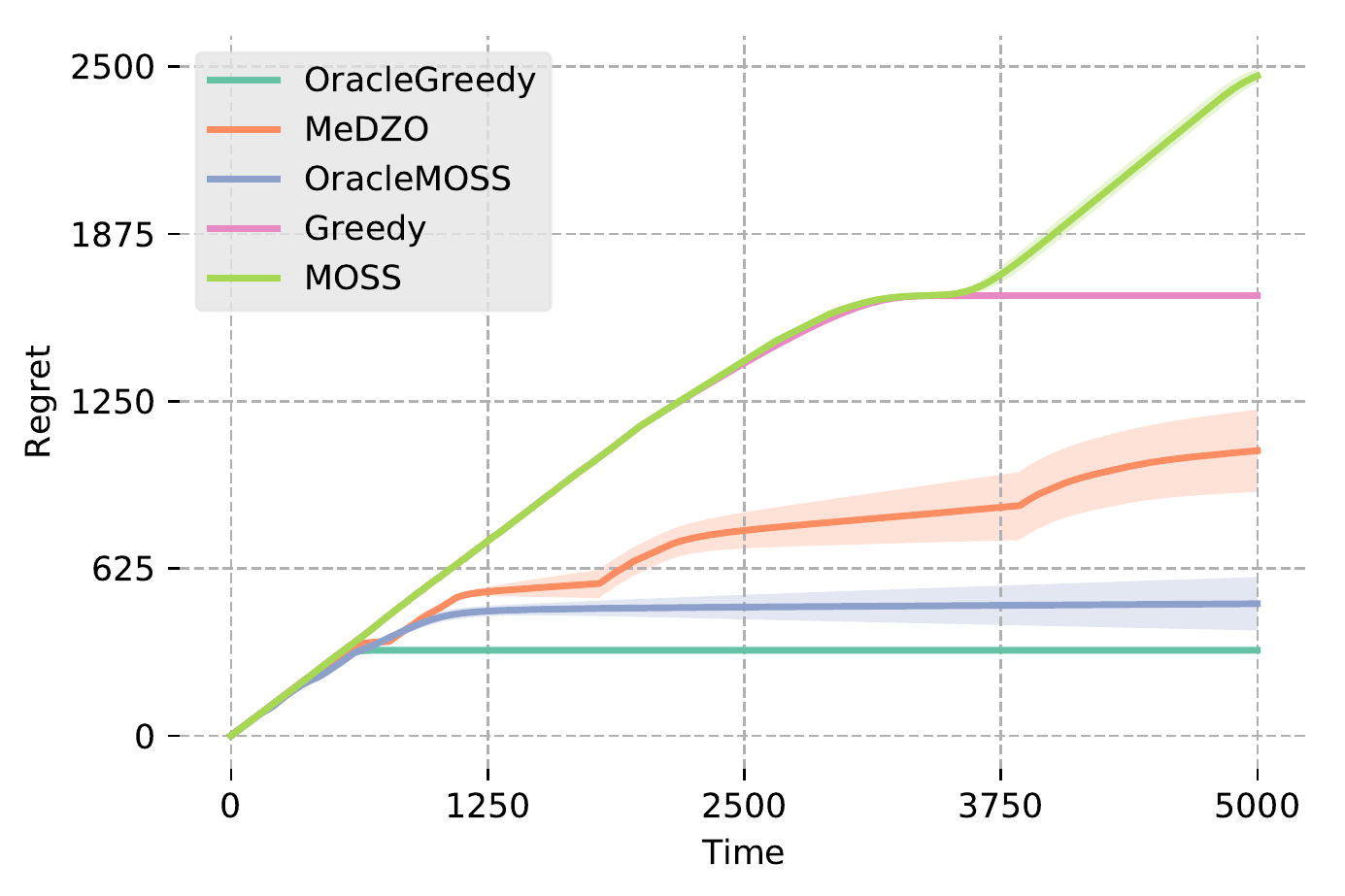}
    \caption{Regret of various algorithms on a many-armed bandit problem with hardness  $\gamma=0.4$.}
    \label{fig:many_armed}
\end{figure}

Once again we observe the excellent performance of \Greedy{} on a subsampling of arms; it outperforms \textsc{OracleMOSS}, its closest competitor, since both assume the knowledge of the hardness parameter $\gamma$ and subsample. It is also interesting to notice that the variance of \textsc{OracleGreedy} is much smaller than \textsc{OracleMOSS}.

\section{Conclusion}

In this paper, we have refined the standard version of \Greedy{} by considering a subsampling of arms and proved sublinear worst-case regret bounds in several bandit models. We also carried out an extensive experimental evaluation which reveals that it outperforms the state-of-the-art for relatively short time horizon. 
Besides, since its indexes are usually computed by most algorithms, it is trivial to implement and fast to run.
Consequently, the \Greedy{} algorithm should be considered as a standard baseline when multiple near-optimal arms are present, which is the case in many models as we saw.

\paragraph{Interesting Direction} We leave open the question of adaptivity.
Adaptivity here could refer to adaptive subsampling or adaptivity to unknown parameters. In particular in the continuous-armed bandit problem, previous work showed that the learner pays a polynomial cost to adapt \citep{hadiji2019polynomial}. Knowing that \Greedy{} works best for relatively short time horizon, it might be interesting to study this cost for a greedy strategy and for what time horizon it might be worth it. 

Another interesting, and relevant in practical problems, direction is to analyze the performance of \Greedy{} in combinatorial bandits (with a large number of  arms and thus a non-tractable number of \textsl{actions}), but with some structure on the rewards on arms \citep{DegenneCB,Perrault2,Perrault1}.

\acks{The research presented was supported by the French National Research Agency, under the project BOLD (ANR19-CE23-0026-04) and it was also supported in part by a public grant as part of the Investissement d'avenir project, reference ANR-11-LABX-0056-LMH, LabEx LMH, in a joint call with Gaspard Monge Program for optimization, operations research and their interactions with data sciences.
}

\bibliography{references}

\begin{thebibliography}{55}
\providecommand{\natexlab}[1]{#1}
\providecommand{\url}[1]{\texttt{#1}}
\expandafter\ifx\csname urlstyle\endcsname\relax
  \providecommand{\doi}[1]{doi: #1}\else
  \providecommand{\doi}{doi: \begingroup \urlstyle{rm}\Url}\fi

\bibitem[Abbasi-Yadkori et~al.(2011)Abbasi-Yadkori, P{\'a}l, and
  Szepesv{\'a}ri]{abbasi2011improved}
Yasin Abbasi-Yadkori, D{\'a}vid P{\'a}l, and Csaba Szepesv{\'a}ri.
\newblock Improved algorithms for linear stochastic bandits.
\newblock In \emph{Advances in Neural Information Processing Systems}, pages
  2312--2320, 2011.

\bibitem[Agrawal(1995)]{agrawal1995continuum}
Rajeev Agrawal.
\newblock The continuum-armed bandit problem.
\newblock \emph{SIAM journal on control and optimization}, 33\penalty0
  (6):\penalty0 1926--1951, 1995.

\bibitem[Agrawal and Goyal(2012)]{agrawal2012analysis}
Shipra Agrawal and Navin Goyal.
\newblock Analysis of thompson sampling for the multi-armed bandit problem.
\newblock In \emph{Conference on learning theory}, pages 39--1, 2012.

\bibitem[Audibert and Bubeck(2009)]{audibert2009minimax}
Jean-Yves Audibert and S{\'e}bastien Bubeck.
\newblock Minimax policies for adversarial and stochastic bandits.
\newblock 2009.

\bibitem[Auer et~al.(2002)Auer, Cesa-Bianchi, and Fischer]{auer2002finite}
Peter Auer, Nicolo Cesa-Bianchi, and Paul Fischer.
\newblock Finite-time analysis of the multiarmed bandit problem.
\newblock \emph{Machine learning}, 47\penalty0 (2-3):\penalty0 235--256, 2002.

\bibitem[Auer et~al.(2007)Auer, Ortner, and Szepesv{\'a}ri]{auer2007improved}
Peter Auer, Ronald Ortner, and Csaba Szepesv{\'a}ri.
\newblock Improved rates for the stochastic continuum-armed bandit problem.
\newblock In \emph{International Conference on Computational Learning Theory},
  pages 454--468. Springer, 2007.

\bibitem[Baransi et~al.(2014)Baransi, Maillard, and Mannor]{baransi2014sub}
Akram Baransi, Odalric-Ambrym Maillard, and Shie Mannor.
\newblock Sub-sampling for multi-armed bandits.
\newblock In \emph{Joint European Conference on Machine Learning and Knowledge
  Discovery in Databases}, pages 115--131. Springer, 2014.

\bibitem[Bastani et~al.(2017)Bastani, Bayati, and Khosravi]{bastani2017mostly}
Hamsa Bastani, Mohsen Bayati, and Khashayar Khosravi.
\newblock Mostly exploration-free algorithms for contextual bandits.
\newblock \emph{arXiv preprint arXiv:1704.09011}, 2017.

\bibitem[Bayati et~al.(2020)Bayati, Hamidi, Johari, and
  Khosravi]{bayati2020optimal}
Mohsen Bayati, Nima Hamidi, Ramesh Johari, and Khashayar Khosravi.
\newblock Unreasonable effectiveness of greedy algorithms in multi-armed bandit
  with many arms.
\newblock \emph{Advances in Neural Information Processing Systems}, 33, 2020.

\bibitem[Berry et~al.(1997)Berry, Chen, Zame, Heath, and
  Shepp]{berry1997bandit}
Donald~A Berry, Robert~W Chen, Alan Zame, David~C Heath, and Larry~A Shepp.
\newblock Bandit problems with infinitely many arms.
\newblock \emph{The Annals of Statistics}, pages 2103--2116, 1997.

\bibitem[Besson and Kaufmann(2018)]{besson2018doubling}
Lilian Besson and Emilie Kaufmann.
\newblock What doubling tricks can and can't do for multi-armed bandits.
\newblock \emph{arXiv preprint arXiv:1803.06971}, 2018.

\bibitem[Bietti et~al.(2018)Bietti, Agarwal, and
  Langford]{bietti2018contextual}
Alberto Bietti, Alekh Agarwal, and John Langford.
\newblock A contextual bandit bake-off.
\newblock \emph{arXiv preprint arXiv:1802.04064}, 2018.

\bibitem[Bonald and Proutiere(2013)]{bonald2013two}
Thomas Bonald and Alexandre Proutiere.
\newblock Two-target algorithms for infinite-armed bandits with bernoulli
  rewards.
\newblock In \emph{Advances in Neural Information Processing Systems}, pages
  2184--2192, 2013.

\bibitem[Bubeck and Cesa-Bianchi(2012)]{bubeck2012regret}
S{\'e}bastien Bubeck and Nicolo Cesa-Bianchi.
\newblock Regret analysis of stochastic and nonstochastic multi-armed bandit
  problems.
\newblock \emph{arXiv preprint arXiv:1204.5721}, 2012.

\bibitem[Bubeck et~al.(2010)Bubeck, Munos, Stoltz, and Szepesvari]{bubeck2010x}
S{\'e}bastien Bubeck, R{\'e}mi Munos, Gilles Stoltz, and Csaba Szepesvari.
\newblock X-armed bandits.
\newblock \emph{arXiv preprint arXiv:1001.4475}, 2010.

\bibitem[Bubeck et~al.(2011)Bubeck, Stoltz, and Yu]{bubeck2011lipschitz}
S{\'e}bastien Bubeck, Gilles Stoltz, and Jia~Yuan Yu.
\newblock Lipschitz bandits without the lipschitz constant.
\newblock In \emph{International Conference on Algorithmic Learning Theory},
  pages 144--158. Springer, 2011.

\bibitem[Carpentier and Valko(2015)]{carpentier2015simple}
Alexandra Carpentier and Michal Valko.
\newblock Simple regret for infinitely many armed bandits.
\newblock In \emph{International Conference on Machine Learning}, pages
  1133--1141, 2015.

\bibitem[Chakrabarti et~al.(2009)Chakrabarti, Kumar, Radlinski, and
  Upfal]{chakrabarti2009mortal}
Deepayan Chakrabarti, Ravi Kumar, Filip Radlinski, and Eli Upfal.
\newblock Mortal multi-armed bandits.
\newblock In \emph{Advances in neural information processing systems}, pages
  273--280, 2009.

\bibitem[Chaudhuri and Kalyanakrishnan(2018)]{chaudhuri2018quantile}
Arghya~Roy Chaudhuri and Shivaram Kalyanakrishnan.
\newblock Quantile-regret minimisation in infinitely many-armed bandits.
\newblock In \emph{UAI}, pages 425--434, 2018.

\bibitem[Cheung et~al.(2019)Cheung, Tan, and Zhong]{cheung2019thompson}
Wang~Chi Cheung, Vincent Tan, and Zixin Zhong.
\newblock A thompson sampling algorithm for cascading bandits.
\newblock In \emph{The 22nd International Conference on Artificial Intelligence
  and Statistics}, pages 438--447, 2019.

\bibitem[Degenne and Perchet(2016{\natexlab{a}})]{DegenneCB}
R{\'e}my Degenne and Vianney Perchet.
\newblock Combinatorial semi-bandit with known covariance.
\newblock In \emph{Advances in Neural Information Processing Systems}, pages
  2972--2980, 2016{\natexlab{a}}.

\bibitem[Degenne and Perchet(2016{\natexlab{b}})]{DegenneMoss}
Rémy Degenne and Vianney Perchet.
\newblock Anytime optimal algorithms in stochastic multi-armed bandits.
\newblock In Maria~Florina Balcan and Kilian~Q. Weinberger, editors,
  \emph{Proceedings of The 33rd International Conference on Machine Learning},
  volume~48 of \emph{Proceedings of Machine Learning Research}, pages
  1587--1595, New York, New York, USA, 20--22 Jun 2016{\natexlab{b}}. PMLR.
\newblock URL \url{http://proceedings.mlr.press/v48/degenne16.html}.

\bibitem[Deshpande and Montanari(2012)]{deshpande2012linear}
Yash Deshpande and Andrea Montanari.
\newblock Linear bandits in high dimension and recommendation systems.
\newblock In \emph{2012 50th Annual Allerton Conference on Communication,
  Control, and Computing (Allerton)}, pages 1750--1754. IEEE, 2012.

\bibitem[Even-Dar et~al.(2002)Even-Dar, Mannor, and Mansour]{even2002pac}
Eyal Even-Dar, Shie Mannor, and Yishay Mansour.
\newblock Pac bounds for multi-armed bandit and markov decision processes.
\newblock In \emph{International Conference on Computational Learning Theory},
  pages 255--270. Springer, 2002.

\bibitem[Garivier et~al.(2019)Garivier, M{\'e}nard, and
  Stoltz]{garivier2019explore}
Aur{\'e}lien Garivier, Pierre M{\'e}nard, and Gilles Stoltz.
\newblock Explore first, exploit next: The true shape of regret in bandit
  problems.
\newblock \emph{Mathematics of Operations Research}, 44\penalty0 (2):\penalty0
  377--399, 2019.

\bibitem[Hadiji(2019)]{hadiji2019polynomial}
H{\'e}di Hadiji.
\newblock Polynomial cost of adaptation for x-armed bandits.
\newblock In \emph{Advances in Neural Information Processing Systems}, pages
  1029--1038, 2019.

\bibitem[Honda and Takemura(2010)]{honda2010asymptotically}
Junya Honda and Akimichi Takemura.
\newblock An asymptotically optimal bandit algorithm for bounded support
  models.
\newblock In \emph{COLT}, pages 67--79. Citeseer, 2010.

\bibitem[Honda and Takemura(2015)]{honda2015non}
Junya Honda and Akimichi Takemura.
\newblock Non-asymptotic analysis of a new bandit algorithm for semi-bounded
  rewards.
\newblock \emph{The Journal of Machine Learning Research}, 16\penalty0
  (1):\penalty0 3721--3756, 2015.

\bibitem[Jedor et~al.(2020)Jedor, Lou{\"e}dec, and Perchet]{jedor2020lifelong}
Matthieu Jedor, Jonathan Lou{\"e}dec, and Vianney Perchet.
\newblock Lifelong learning in multi-armed bandits.
\newblock \emph{arXiv preprint arXiv:2012.14264}, 2020.

\bibitem[Kannan et~al.(2018)Kannan, Morgenstern, Roth, Waggoner, and
  Wu]{kannan2018smoothed}
Sampath Kannan, Jamie~H Morgenstern, Aaron Roth, Bo~Waggoner, and Zhiwei~Steven
  Wu.
\newblock A smoothed analysis of the greedy algorithm for the linear contextual
  bandit problem.
\newblock In \emph{Advances in Neural Information Processing Systems}, pages
  2227--2236, 2018.

\bibitem[Kleinberg et~al.(2008)Kleinberg, Slivkins, and
  Upfal]{kleinberg2008multi}
Robert Kleinberg, Aleksandrs Slivkins, and Eli Upfal.
\newblock Multi-armed bandits in metric spaces.
\newblock In \emph{Proceedings of the fortieth annual ACM symposium on Theory
  of computing}, pages 681--690, 2008.

\bibitem[Kleinberg(2005)]{kleinberg2005nearly}
Robert~D Kleinberg.
\newblock Nearly tight bounds for the continuum-armed bandit problem.
\newblock In \emph{Advances in Neural Information Processing Systems}, pages
  697--704, 2005.

\bibitem[Kuleshov and Precup(2014)]{kuleshov2014algorithms}
Volodymyr Kuleshov and Doina Precup.
\newblock Algorithms for multi-armed bandit problems.
\newblock \emph{arXiv preprint arXiv:1402.6028}, 2014.

\bibitem[Kveton et~al.(2015)Kveton, Szepesvari, Wen, and
  Ashkan]{kveton2015cascading}
Branislav Kveton, Csaba Szepesvari, Zheng Wen, and Azin Ashkan.
\newblock Cascading bandits: Learning to rank in the cascade model.
\newblock In \emph{International Conference on Machine Learning}, pages
  767--776, 2015.

\bibitem[Lattimore and Szepesv{\'a}ri(2020)]{lattimore2018bandit}
Tor Lattimore and Csaba Szepesv{\'a}ri.
\newblock \emph{Bandit algorithms}.
\newblock Cambridge University Press, 2020.

\bibitem[Locatelli and Carpentier(2018)]{locatelli2018adaptivity}
Andrea Locatelli and Alexandra Carpentier.
\newblock Adaptivity to smoothness in x-armed bandits.
\newblock In \emph{Conference on Learning Theory}, pages 1463--1492, 2018.

\bibitem[Mersereau et~al.(2009)Mersereau, Rusmevichientong, and
  Tsitsiklis]{mersereau2009structured}
Adam~J Mersereau, Paat Rusmevichientong, and John~N Tsitsiklis.
\newblock A structured multiarmed bandit problem and the greedy policy.
\newblock \emph{IEEE Transactions on Automatic Control}, 54\penalty0
  (12):\penalty0 2787--2802, 2009.

\bibitem[Ou et~al.(2019)Ou, Li, Yang, Zhu, and Jin]{ou2019semi}
Mingdong Ou, Nan Li, Cheng Yang, Shenghuo Zhu, and Rong Jin.
\newblock Semi-parametric sampling for stochastic bandits with many arms.
\newblock In \emph{Proceedings of the AAAI Conference on Artificial
  Intelligence}, volume~33, pages 7933--7940, 2019.

\bibitem[Perchet and Rigollet(2013)]{perchet2013}
Vianney Perchet and Philippe Rigollet.
\newblock The multi-armed bandit problem with covariates.
\newblock \emph{Ann. Statist.}, 41\penalty0 (2):\penalty0 693--721, 04 2013.
\newblock \doi{10.1214/13-AOS1101}.
\newblock URL \url{https://doi.org/10.1214/13-AOS1101}.

\bibitem[Perchet et~al.(2016)Perchet, Rigollet, Chassang, and
  Snowberg]{perchet2016}
Vianney Perchet, Philippe Rigollet, Sylvain Chassang, and Erik Snowberg.
\newblock Batched bandit problems.
\newblock \emph{Ann. Statist.}, 44\penalty0 (2):\penalty0 660--681, 04 2016.
\newblock \doi{10.1214/15-AOS1381}.
\newblock URL \url{https://doi.org/10.1214/15-AOS1381}.

\bibitem[Perrault et~al.(2019)Perrault, Perchet, and Valko]{Perrault2}
Pierre Perrault, Vianney Perchet, and Michal Valko.
\newblock Exploiting structure of uncertainty for efficient matroid
  semi-bandits.
\newblock \emph{arXiv preprint arXiv:1902.03794}, 2019.

\bibitem[Perrault et~al.(2020{\natexlab{a}})Perrault, Boursier, Valko, and
  Perchet]{Perrault3}
Pierre Perrault, Etienne Boursier, Michal Valko, and Vianney Perchet.
\newblock Statistical efficiency of thompson sampling for combinatorial
  semi-bandits.
\newblock \emph{Advances in Neural Information Processing Systems}, 33,
  2020{\natexlab{a}}.

\bibitem[Perrault et~al.(2020{\natexlab{b}})Perrault, Valko, and
  Perchet]{Perrault1}
Pierre Perrault, Michal Valko, and Vianney Perchet.
\newblock Covariance-adapting algorithm for semi-bandits with application to
  sparse outcomes.
\newblock In \emph{Conference on Learning Theory}, pages 3152--3184. PMLR,
  2020{\natexlab{b}}.

\bibitem[Raghavan et~al.(2018)Raghavan, Slivkins, Vaughan, and
  Wu]{raghavan2018externalities}
Manish Raghavan, Aleksandrs Slivkins, Jennifer~Wortman Vaughan, and
  Zhiwei~Steven Wu.
\newblock The externalities of exploration and how data diversity helps
  exploitation.
\newblock \emph{arXiv preprint arXiv:1806.00543}, 2018.

\bibitem[Raghavan et~al.(2020)Raghavan, Slivkins, Vaughan, and
  Wu]{raghavan2020greedy}
Manish Raghavan, Aleksandrs Slivkins, Jennifer~Wortman Vaughan, and
  Zhiwei~Steven Wu.
\newblock Greedy algorithm almost dominates in smoothed contextual bandits.
\newblock \emph{arXiv preprint arXiv:2005.10624}, 2020.

\bibitem[Russo and Van~Roy(2018)]{russo2018satisficing}
Daniel Russo and Benjamin Van~Roy.
\newblock Satisficing in time-sensitive bandit learning.
\newblock \emph{arXiv preprint arXiv:1803.02855}, 2018.

\bibitem[Slivkins(2019)]{slivkins2019introduction}
Aleksandrs Slivkins.
\newblock Introduction to multi-armed bandits.
\newblock \emph{arXiv preprint arXiv:1904.07272}, 2019.

\bibitem[Teytaud et~al.(2007)Teytaud, Gelly, and Sebag]{teytaud2007anytime}
Olivier Teytaud, Sylvain Gelly, and Michele Sebag.
\newblock Anytime many-armed bandits.
\newblock 2007.

\bibitem[Thompson(1933)]{thompson1933likelihood}
William~R Thompson.
\newblock On the likelihood that one unknown probability exceeds another in
  view of the evidence of two samples.
\newblock \emph{Biometrika}, 25\penalty0 (3/4):\penalty0 285--294, 1933.

\bibitem[Vermorel and Mohri(2005)]{vermorel2005multi}
Joannes Vermorel and Mehryar Mohri.
\newblock Multi-armed bandit algorithms and empirical evaluation.
\newblock In \emph{European conference on machine learning}, pages 437--448.
  Springer, 2005.

\bibitem[Wang et~al.(2017)Wang, Kurniawati, and Kroese]{wang2017cemab}
Erli Wang, Hanna Kurniawati, and Dirk~P Kroese.
\newblock Cemab: A cross-entropy-based method for large-scale multi-armed
  bandits.
\newblock In \emph{Australasian Conference on Artificial Life and Computational
  Intelligence}, pages 353--365. Springer, 2017.

\bibitem[Wang et~al.(2009)Wang, Audibert, and Munos]{wang2009algorithms}
Yizao Wang, Jean-Yves Audibert, and R{\'e}mi Munos.
\newblock Algorithms for infinitely many-armed bandits.
\newblock In \emph{Advances in Neural Information Processing Systems}, pages
  1729--1736, 2009.

\bibitem[Xia et~al.(2015)Xia, Li, Qin, Yu, and Liu]{xia2015thompson}
Yingce Xia, Haifang Li, Tao Qin, Nenghai Yu, and Tie-Yan Liu.
\newblock Thompson sampling for budgeted multi-armed bandits.
\newblock In \emph{Twenty-Fourth International Joint Conference on Artificial
  Intelligence}, 2015.

\bibitem[Xia et~al.(2016)Xia, Ding, Zhang, Yu, and Qin]{xia2016budgeted}
Yingce Xia, Wenkui Ding, Xu-Dong Zhang, Nenghai Yu, and Tao Qin.
\newblock Budgeted bandit problems with continuous random costs.
\newblock In \emph{Asian conference on machine learning}, pages 317--332, 2016.

\bibitem[Zhu and Nowak(2020)]{zhu2020multiple}
Yinglun Zhu and Robert Nowak.
\newblock On regret with multiple best arms.
\newblock \emph{Advances in Neural Information Processing Systems}, 33, 2020.

\end{thebibliography}

\newpage

\appendix

\section{Additional figures}

This section provides illustrations that we did not include in the article in order to not overload it.

\subsection{Failure of \Greedy{}}
\label{sub:failure}

Here, we illustrate Example \ref{ex:failure}, that is the failure of \Greedy{}. We recall that we considered a Bernoulli bandit problem consisting of $K=2$ arms with mean rewards $0.9$ and $0.1$ respectively. In Figure \ref{fig:failure}, we compare the regret of \Greedy{} with the \textsc{Thompson Sampling} algorithm \citep{thompson1933likelihood}.

\begin{figure}
    \centering
    \includegraphics[width=0.5\linewidth]{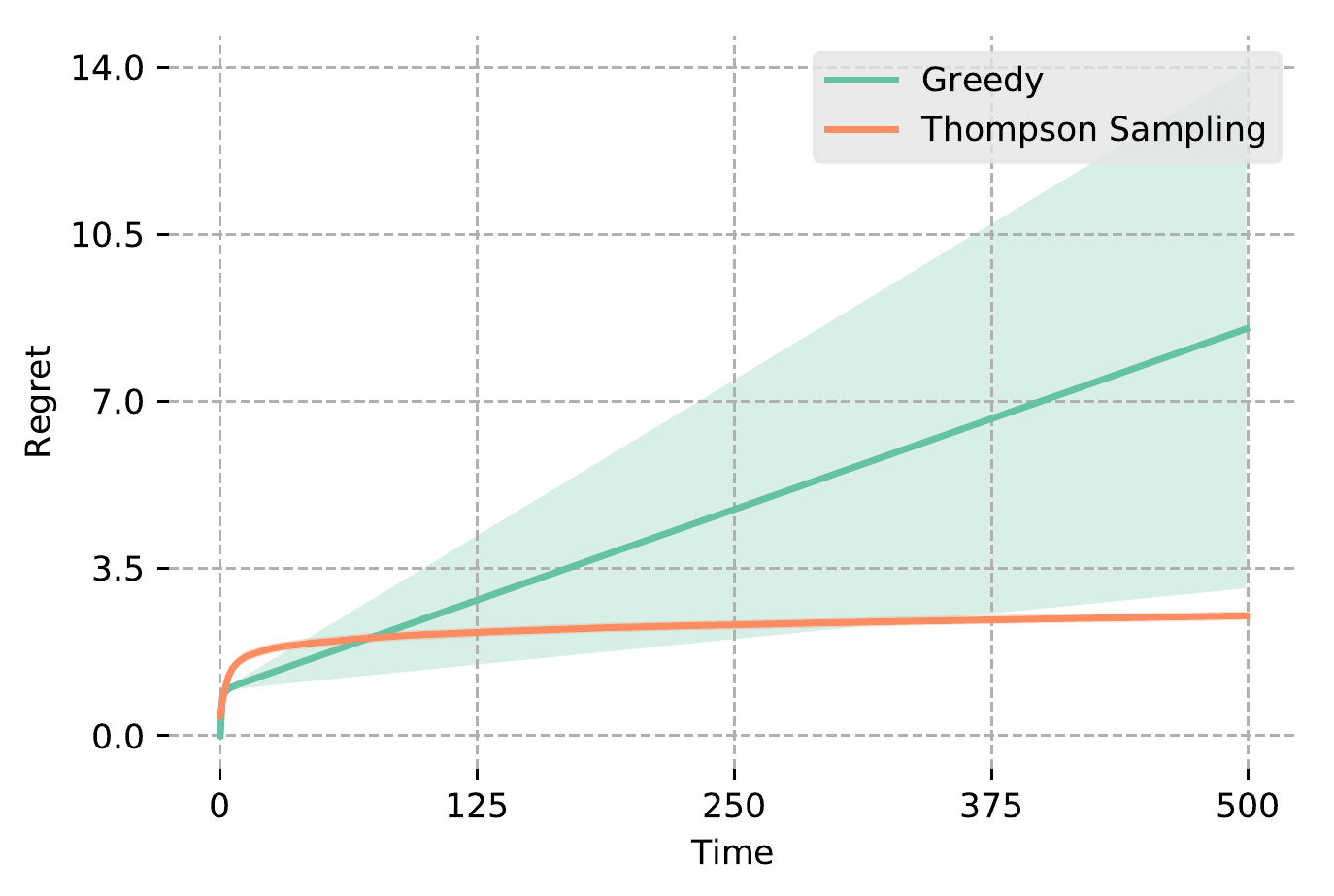}
    \caption{Regret of various algorithms as a function of time in a Bernoulli bandit problem. Results are averaged over $1000$ runs and the shaded area represents $0.1$ standard deviation.}
    \label{fig:failure}
\end{figure}

\subsection{Continuous functions studied}
\label{sub:continuous}

This section provides the plots of the studied functions in Subsection \ref{sec:exp_continuous}. These functions, displayed on Figure \ref{fig:continuous_functions}, are recalled below for convenience.

\begin{align*}
    f_1 : x &\mapsto 0.5 \sin(13x)\sin(27x) + 0.5, \\
    f_2 : x &\mapsto\max\left( 3.6 x (1-x), 1 - | x - 0.05 | / 0.05 \right)  \\  
    f_3 : x &\mapsto x (1-x) \left( 4 - \sqrt{| \sin{60x} |} \right)
\end{align*}

\begin{figure}
  \floatconts
  {fig:continuous_functions}
  {\caption{Functions considered in the continuous-armed bandit experiments.}}
  {%
    \subfigure[$f_1$]{
      \includegraphics[width=0.33\linewidth]{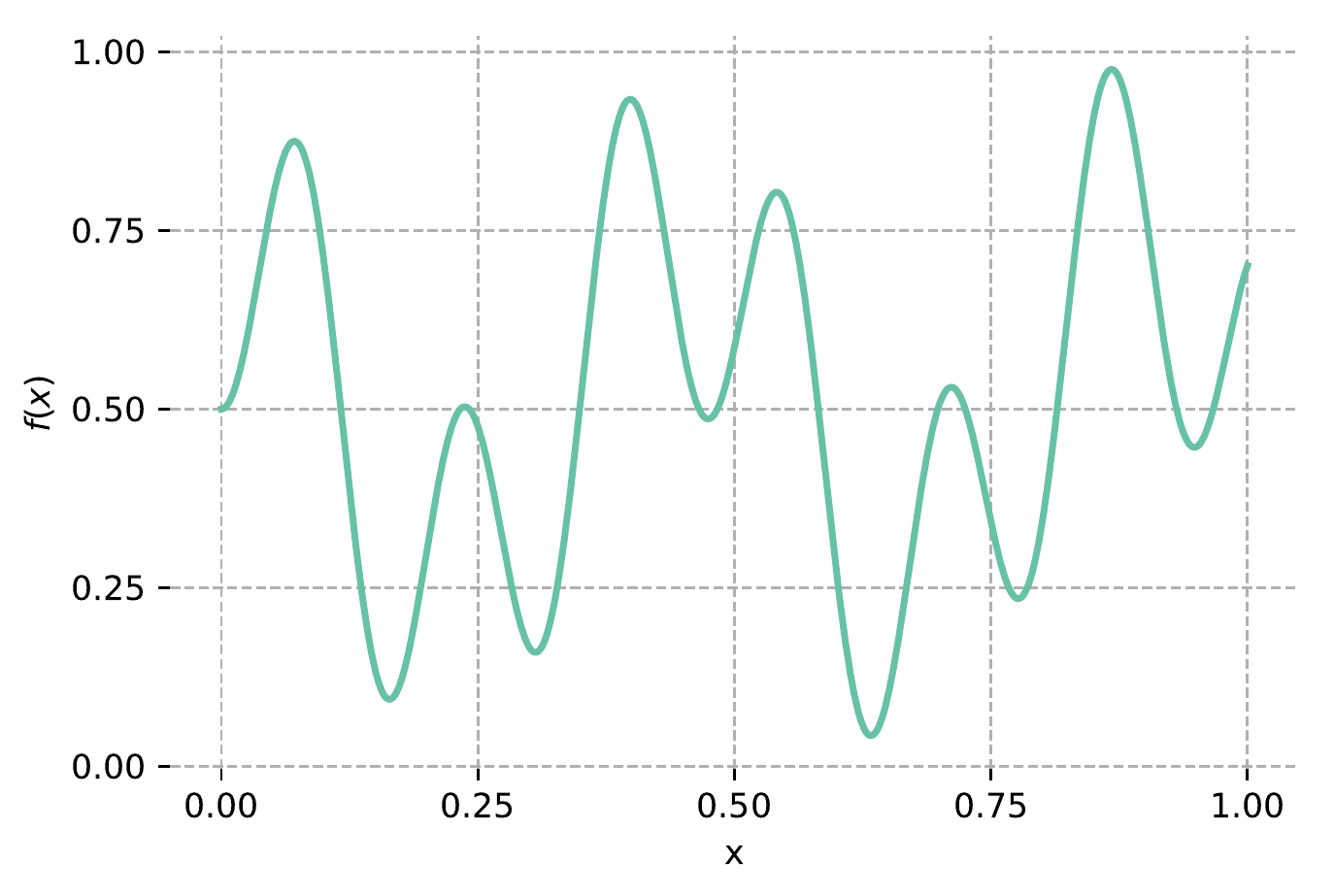}}%
    \subfigure[$f_2$]{
      \includegraphics[width=0.33\linewidth]{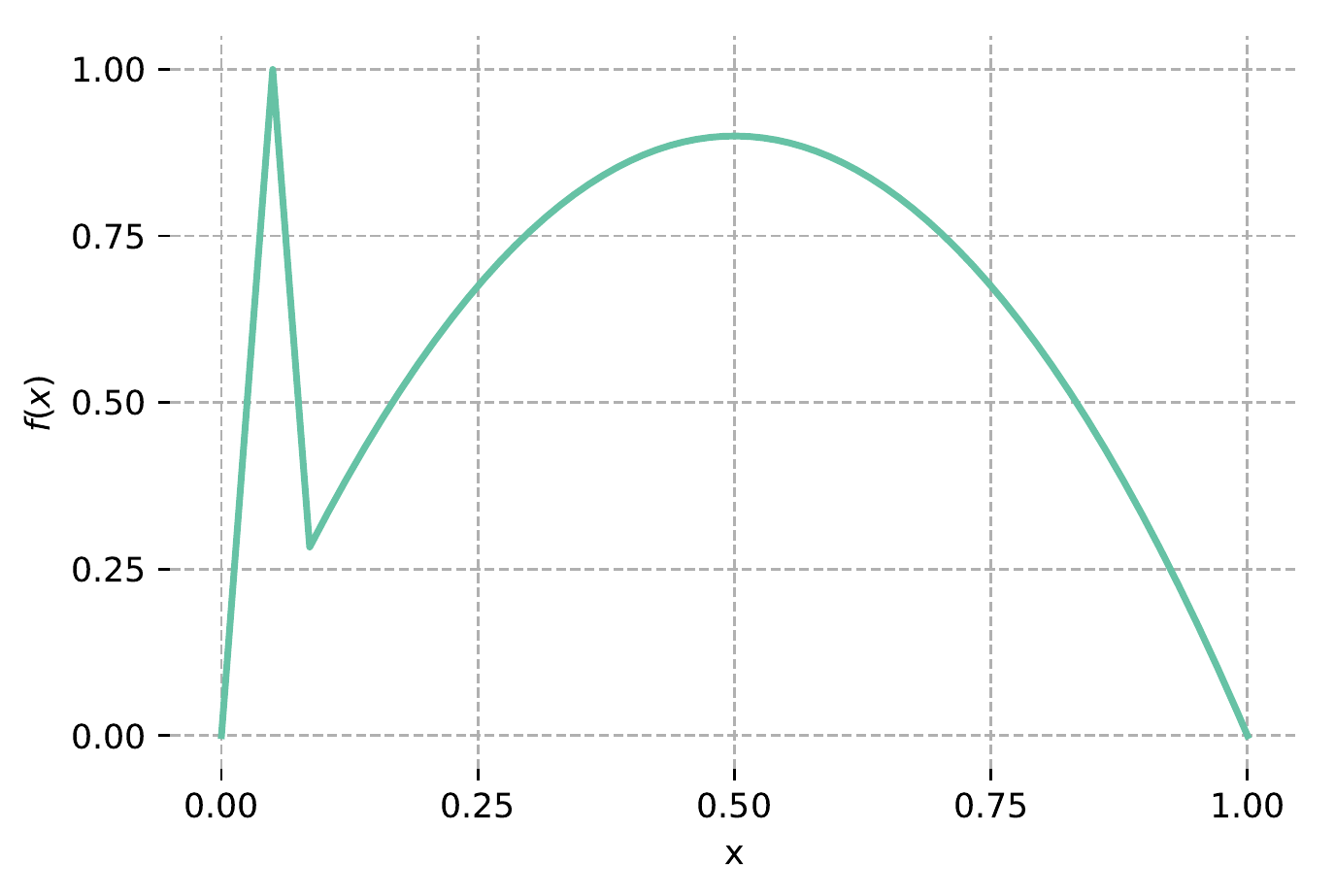}}
    \subfigure[$f_3$]{
      \includegraphics[width=0.33\linewidth]{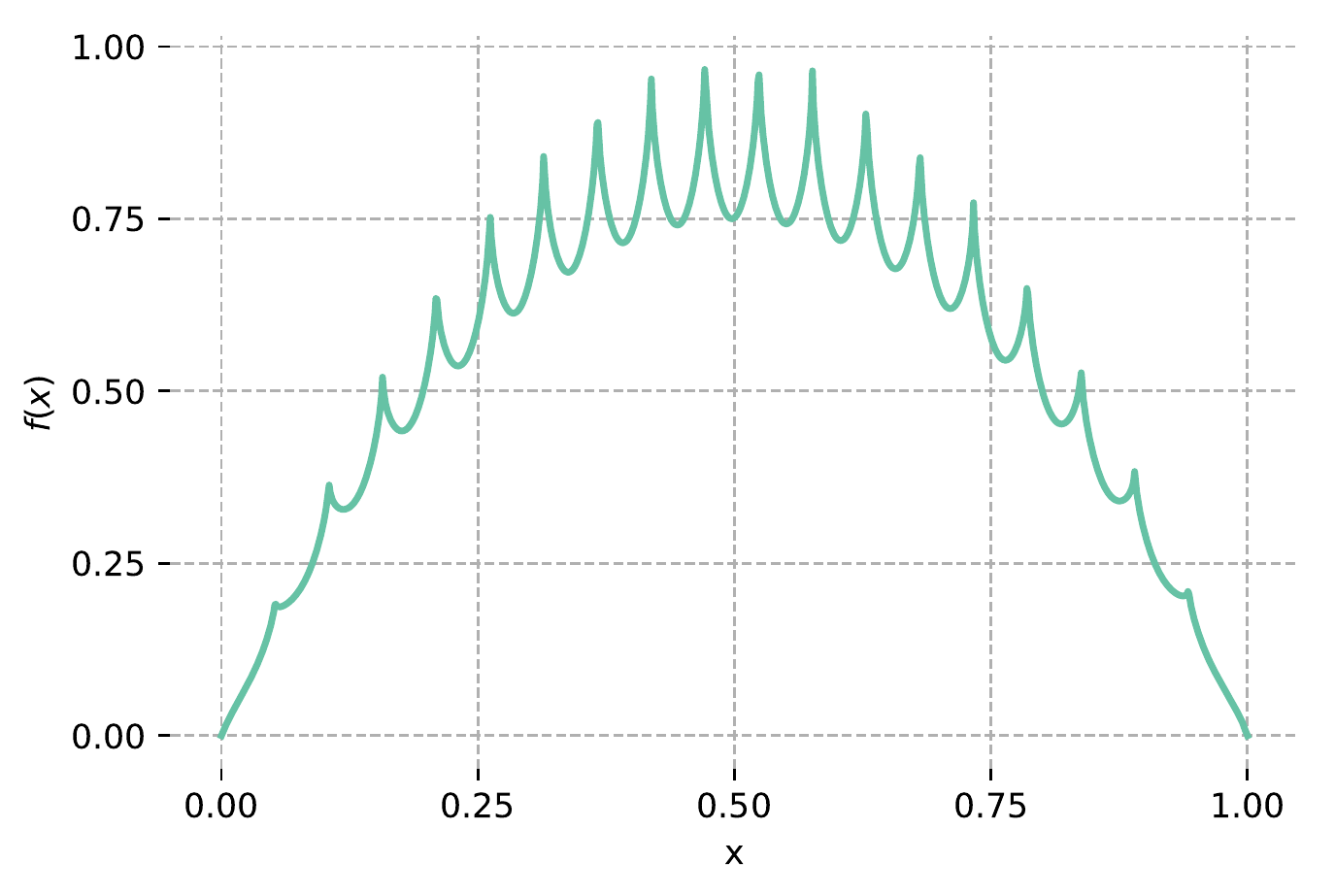}}
  }
\end{figure}

\section{Proofs of Section \ref{sec:generic}}
\subsection{Proof of Theorem \ref{thm:generic_bound_greedy}}
\label{app:proof_generic_bound}

The proof combines two techniques standard in the literature: creating a ``good'' event in order to distinguish the randomness of the distributions from the behavior of the algorithm and decomposing the arms into near-optimal and suboptimal ones.
Fix some $\varepsilon > 0$.

\paragraph{Good event}
Define the event $\mathfrak{E}$, through its complement, by
\begin{equation*}
   \mathfrak{E}^c = \bigcap_{k: \Delta_k \leq \varepsilon} \mathbb{P}\left( \exists t\, |\, \widehat{\mu}_k(t) \leq \mu_k - \varepsilon \right) \,.
\end{equation*}
In words, $\mathfrak{E}$ is the event that at least one $\varepsilon$-optimal arm is never underestimated by more than $\varepsilon$ below its mean reward. Using the independence of the events along with the concentration bound of \citet{bayati2020optimal}, see Lemma \ref{lemma:concentration_bound}, we obtain
\begin{equation}
    \mathbb{P}\left(\mathfrak{E}^c\right) \leq \exp\left( - N_\varepsilon \frac{\varepsilon^2}{2} \right) \,.
    \label{eq:proof_bad_event}
\end{equation}

\paragraph{Bound on the number of pulls of suboptimal arms}
On the event $\mathfrak{E}$, let $k \in [K]$ be an arm such that $\Delta_k > 3 \varepsilon$. With a slight abuse of notation, we denote by $\widehat{\mu}_k^t$ the average reward of arm $k$ after $t$ samples. The expected number of pulls of arm $k$ is then bounded by
\begin{align}
    \mathbb{E}\left[ N_k(T)\, |\, \mathfrak{E} \right] &\leq 1 + \sum_{t=1}^\infty \mathbb{P}\left( \widehat{\mu}_k^t \geq \mu^\star - 2\varepsilon \right) \nonumber \\ 
    &= 1 + \sum_{t=1}^\infty \mathbb{P}\left( \widehat{\mu}_k^t - \mu_k \geq \Delta_k - 2\varepsilon \right) \nonumber \\
    &\leq 1 + \sum_{t=1}^\infty \exp\left(- t \frac{(\Delta_k - 2\varepsilon)^2}{2} \right) \nonumber \\
    &= 1 + \frac{1}{\exp\left( \frac{(\Delta_k - 2\varepsilon)^2}{2} \right) - 1} \nonumber \\
    &\leq 1 + \frac{2}{(\Delta_k - 2\varepsilon)^2}
    \label{eq:proof_nb_pulls_sub}
\end{align}
where in second inequality we use Lemma \ref{lemma:concentration_bound_mean_subgauss} since $\widehat{\mu}_k^t$ is $1/t$-subgaussian and in the last inequality we used that $e^x \geq 1+x$ for all $x \in \mathbb{R}$.

\paragraph{Putting things together} We first decompose the regret according to the event $\mathfrak{E}$
\begin{equation}
    R_T \leq \mathbb{E}\left[R_T | \mathfrak{E}^c \right] \mathbb{P}\left(\mathfrak{E}^c\right)  + \mathbb{E}\left[R_T | \mathfrak{E} \right] \,.
    \label{eq:proof_decomp_event}
\end{equation}
As mean rewards are bounded in $[0, 1]$, the regret on the bad event is bounded by $T$ and by Equation \eqref{eq:proof_bad_event} we have
\begin{equation*}
    \mathbb{E}\left[R_T | \mathfrak{E}^c \right] \mathbb{P}\left(\mathfrak{E}^c\right) \leq T \exp\left( - N_\varepsilon \frac{\varepsilon^2}{2} \right) \,.
\end{equation*}
We further decompose the second term on the right-hand side of Equation \eqref{eq:proof_decomp_event},
\begin{equation*}
    \mathbb{E}\left[R_T | \mathfrak{E} \right] \leq \sum_{k:\Delta_k \leq 3 \varepsilon} \Delta_k \mathbb{E}\left[ N_k(T) | \mathfrak{E} \right] + \sum_{k:\Delta_k > 3 \varepsilon} \Delta_k \mathbb{E}\left[ N_k(T) | \mathfrak{E} \right] \,.
\end{equation*}
The first term is trivially bounded by $3\varepsilon T$, while for the second term we have by Equation \eqref{eq:proof_nb_pulls_sub}, 
\begin{align*}
    \sum_{k:\Delta_k > 3 \varepsilon} \Delta_k \mathbb{E}\left[ N_k(T) | \mathfrak{E} \right] &\leq \sum_{k:\Delta_k > 3 \varepsilon} \frac{2 \Delta_k}{(\Delta_k - 2\varepsilon)^2} + \sum_{k=1}^K \Delta_k \\
    &\leq \sum_{k:\Delta_k > 3 \varepsilon} \frac{6}{(\Delta_k - 2\varepsilon)} + \sum_{k=1}^K \Delta_k \\
    &\leq \sum_{k:\Delta_k > 3 \varepsilon} \frac{6}{\varepsilon} + \sum_{k=1}^K \Delta_k \leq \frac{6K}{\varepsilon} + \sum_{k=1}^K \Delta_k
\end{align*}
where in the second inequality we used that $\Delta_k \leq 3 (\Delta_k - 2 \varepsilon)$, which holds true since $\Delta_k \geq 3\varepsilon$.
Hence the result.
\subsection{Proof of Corollary \ref{CR:chaining}}
\label{CR:proof_chaining}

We recall the definition of the event $\mathfrak{E}_\varepsilon$, through its complement $\mathfrak{E}^c_\varepsilon$,
\begin{equation*}
    \mathfrak{E}^c_\varepsilon = \bigcap_{k: \Delta_k \leq \varepsilon} \mathbb{P}\left( \exists t\, |\, \widehat{\mu}_k(t) \leq \mu_k - \varepsilon \right) \,.
\end{equation*}
Consider any increasing sequence $\left\{ \varepsilon_m \right\}_{m=0}^M$ and denote $\mathfrak{E}_m$ the good event associated with $\varepsilon_m$ for $m \in \{0, \ldots, M \}$. By the chain rule and the previous computation of the regret on the good event (see proof of Theorem \ref{thm:generic_bound_greedy}), we have
\begin{align*}
    R_T &\leq \left( 3 \varepsilon_0 T + \frac{6 K}{\varepsilon_0} \right) \mathbb{P}\left( \mathfrak{E}_0 \right) + \left( 3 \varepsilon_1 T + \frac{6 K}{\varepsilon_1} \right) \mathbb{P}( \mathfrak{E}_1 \cap \mathfrak{E}_0^c ) + \ldots \\
    &\quad + \left( 3 \varepsilon_M T + \frac{6 K}{\varepsilon_M} \right) \mathbb{P}( \mathfrak{E}_M \cap \mathfrak{E}_{M-1}^c ) + T \mathbb{P}( \mathfrak{E}_{M-1}^c ) + \sum_{k=1}^K \Delta_k \\
    &\leq \left[\left( 3 \varepsilon_0 T + \frac{6 K}{\varepsilon_0} \right) - \left( 3 \varepsilon_1 T + \frac{6 K}{\varepsilon_1} \right) \right] \mathbb{P}\left( \mathfrak{E}_0 \right) + \ldots \\
    &\quad + \left[\left( 3 \varepsilon_{M-1} T + \frac{6 K}{\varepsilon_{M-1}} \right) - \left( 3 \varepsilon_M T + \frac{6 K}{\varepsilon_M} \right) \right] \mathbb{P}\left( \mathfrak{E}_{M-1} \right) \\ 
    &\quad + \left( 3 \varepsilon_M T + \frac{6 K}{\varepsilon_M} \right) \mathbb{P}\left( \mathfrak{E}_M \right) +  T \mathbb{P}( \mathfrak{E}_{M}^c ) + \sum_{k=1}^K \Delta_k
\end{align*}
where in the second inequality we used that $\mathbf{1}\{\mathfrak{A} \cap \mathfrak{B}^c\} = \mathbf{1}\{\mathfrak{A}\} - \mathbf{1}\{\mathfrak{B}\}$ if $\mathfrak{B} \subset \mathfrak{A}$. 
In the proof of Theorem \ref{thm:generic_bound_greedy}, we show that
\begin{equation*}
    \mathbb{P}( \mathfrak{E}_{m}^c ) \leq \exp\left( - N_{\varepsilon_m} \frac{\varepsilon^2_m}{2}\right)
\end{equation*}
for $m \in \{0, \ldots, M \}$. Hence we obtain
\begin{align*}
    R(T) &\leq \left( 3 \varepsilon_0 T + \frac{6 K}{\varepsilon_0} \right) \\
    &\quad + \sum_{m=0}^{M-1} \left[\left( 3 \varepsilon_{m+1} T + \frac{6 K}{\varepsilon_{m+1}} \right) - \left( 3 \varepsilon_m T + \frac{6 K}{\varepsilon_m} \right) \right] \exp\left( - N_{\varepsilon_m} \frac{\varepsilon^2_m}{2} \right)  \\
    &\quad + T \exp\left( - \frac{K}{2} \right) + \sum_{k=1}^K \Delta_k 
\end{align*}
The middle term is upper-bounded by
$$\sum_{m=0}^{M-1}(\varepsilon_{m+1}-\varepsilon_m) \left[3 T + \frac{6 K}{\varepsilon^2_{m}}  \right] \exp\left( - N_{\varepsilon_m} \frac{\varepsilon^2_m}{2} \right),
$$
which converges, as the mesh of the sequence $\varepsilon_m$ goes to zero, towards 
\begin{equation*}
     \int_{\varepsilon}^1 \left(3  T + \frac{6 K}{x^2}\right) \exp\left( - N_x \frac{x^2}{2} \right) dx
\end{equation*}
Hence the result.

\section{Proof of Theorem \ref{thm:continuous}}
\label{app:proof_continuous}

Let $\varepsilon > 0$.
The regret can be decomposed into an approximation and an estimation term,
\begin{equation*}
    T f(x^\star) - \sum_{t=1}^T f(x_t) = T \left(f(x^\star) - \max_{k \in [K]} f\left(\frac{k}{K}\right) \right) + \left( T \max_{k \in [K]} f\left(\frac{k}{K}\right) - \sum_{t=1}^T f(x_t) \right) \,.
\end{equation*}
By Assumption \ref{hyp:continuous}, the first term is bounded by $\varepsilon T$ when $K \geq \left( \frac{L}{\varepsilon}\right)^{1/\alpha}$. Then, according to Theorem \ref{thm:generic_bound_greedy}, we just have to lower bound $N_\varepsilon$ to conclude. To do so, we prove a lower bound on the number of arms that are $\varepsilon$-optimal with respect to the best arm overall. Let $N_\varepsilon^C$ denotes this quantity.

\paragraph{Bound on $N_\varepsilon^C$}
By Assumption \ref{hyp:continuous}, an $\varepsilon$-optimal arm $k$ may verify (there can be $\varepsilon$-optimal that are not around the maxima)
\begin{equation*}
   L \left| x^\star - k/K \right|^\alpha \leq \varepsilon \,.
\end{equation*}
Knowing that $k$ is an integer, we obtain
\begin{equation*}
    \ceil*{K\left( x^\star - \left( \frac{\varepsilon}{L} \right)^{1/\alpha} \right)} \leq k \leq \floor*{K\left( x^\star + \left( \frac{\varepsilon}{L} \right)^{1/\alpha} \right)} \,.
\end{equation*}
This means that we have the following lower bound on $N_\varepsilon^C$
\begin{equation*}
    N_\varepsilon^C \geq \floor*{K\left( x^\star + \left( \frac{\varepsilon}{L} \right)^{1/\alpha} \right)} - \ceil*{K\left( x^\star - \left( \frac{\varepsilon}{L} \right)^{1/\alpha} \right)} + 1 \,.
\end{equation*}
Thanks to Lemma \ref{lemma:1}, we obtain
\begin{equation*}
   N_\varepsilon^C \geq \floor*{2K\left( \varepsilon/L \right)^{1/\alpha}} \,.
\end{equation*}
Finally, using that $\floor*{2x} \geq x$ for $x \geq 1$ (easily verify with the assumption on $K$), we obtain the following lower bound
\begin{equation*}
N_\varepsilon^C \geq K\left( \varepsilon/L \right)^{1/\alpha} \,.
\end{equation*}

\paragraph{Conclusion}
We trivially have that $N_{\varepsilon} \geq N_\varepsilon^C$. 
The first part of the Theorem then results from the fact that $\sum_{k=1}^K \Delta_k \leq K$ since $\mu_k \in [0, 1]$ for all $k \in [K]$. 
On the other hand, the second part comes from taking $\varepsilon^2 = 3K/(2T)$ which is the value of $\varepsilon$ that minimizes the term $4 \varepsilon T + 6 K/\varepsilon$.
\section{Proof of Theorem \ref{thm:infinite}}
\label{app:proof_infinite}

Let $\varepsilon > 0$. Once again, thanks to Theorem \ref{thm:generic_bound_greedy} we just have to bound $N_\varepsilon$ and the result will follow by adding the approximation cost $\varepsilon T$. 
We construct a good event on the expected rewards of sampled arms.
Let $I_\varepsilon = [\mu^\star - \varepsilon, \mu^\star]$ and $N_\varepsilon^I = \sum_{k=1}^K \mathbf{1}\{ k \in I_\varepsilon \}$ be the number of $\varepsilon$-optimal arms with respect to all arms. Assumption \ref{hyp:infinite} implies that
\begin{equation*}
    p = \mathbb{E}\left[ \mathbf{1}\{ k \in I_\varepsilon \} \right] = \mathbb{P}(k \in I_\varepsilon) \in [c_1 \varepsilon^\beta, c_2 \varepsilon^\beta] \,.
\end{equation*}
Let $\delta \in [0, 1)$. By Chernoff inequality we have
\begin{equation*}
   \mathbb{P}\left( N_\varepsilon^I < (1-\delta) Kp \right) \leq \exp\left( - Kp \delta^2/2 \right) \,.
\end{equation*}
In particular, taking $\delta = \frac{1}{2}$ yields
\begin{equation*}
 \mathbb{P}\left( N_\varepsilon^I < c_1 \varepsilon^\beta K/2 \right) \leq \exp\left( - c_1 \varepsilon^\beta K/8 \right) \,.
\end{equation*}
Now we trivially have that $N_\varepsilon \geq N_\varepsilon^I$, and hence we obtain
\begin{equation*}
   \mathbb{P}\left( N_\varepsilon < c_1 \varepsilon^\beta K /2\right) \leq \exp\left( - c_1 \varepsilon^\beta K / 8 \right) \,.
\end{equation*}
By constructing a good event based on the previous concentration bound and using that $\sum_{k=1}^K \Delta_k \leq K$, we obtain the first part of the Theorem.
The second part results from (i) the first exponential term dominates since $\varepsilon^{2+\beta} \leq \varepsilon^\beta$ for all $\varepsilon \in [0, 1]$ and $\beta > 0$ and (ii) the choice of $\varepsilon = \sqrt{3K/(2T)}$ which is the value that minimizes $4\varepsilon T + 6K / \varepsilon$.
\section{Proof of Theorem \ref{thm:many}}
\label{app:proof_many}

Once again, we just need a lower bound on the number of optimal arms in the subsampling and we construct a good event to do so.
We reuse the previous notation $N_\varepsilon$ to denote this value ($\varepsilon = 0$ here). 
Let $N_\varepsilon^S$ be the number of optimal arms with respect to all arms.
In the case of a subsampling of $K$ arms done without replacement, $N_\varepsilon^S$ is distributed according to a hypergeometric distribution.
By Hoeffding's inequality, see Lemma \ref{lemma:hoeffding}, we have for $0 < t < pK$,
\begin{equation*}
    \mathbb{P}\left(N_\varepsilon^S \leq (p-t)K\right) \leq \exp\left(-2t^2K\right)
\end{equation*}
where $p=m/n$.
We want to choose $t$ such $p-t > 0$, otherwise the bound is meaningless. 
In particular, the choice of $t=p/2$ yields
\begin{equation*}
   \mathbb{P}\left(N_\varepsilon^S \leq pK/2 \right) \leq \exp\left(-p^2 K/2\right) \,.
\end{equation*}
We then trivially have that $N_\varepsilon \geq N_\varepsilon^S$.
The regret on the bad events is then given by
\begin{equation*}
    T \left[ \exp\left( - pK\varepsilon^2 / 4\right) + \exp\left( - p^2 K/2\right) \right]
\end{equation*}
For this regret to be $\mathcal{O}(1)$, the two following inequalities must be verify: 
\begin{align*}
    pK \varepsilon^2/4 &\geq \log T \\
    p^2 K/2 &\geq \log T
\end{align*}
Now the term $3 \varepsilon T + \frac{6K}{\varepsilon}$ of Theorem \ref{thm:generic_bound_greedy} is minimized for $\varepsilon^2 = 2K/T$. This leads to
\begin{align*}
    pK^2/2 &\geq T \log T \\
    p^2 K/2 &\geq \log T
\end{align*}
Using that $p=T^{-\alpha}$, we obtain 
\begin{equation*}
   K \geq 2 \max\left\{ \sqrt{T^{1+\alpha}}, T^{2\alpha} \sqrt{\log T} \right\} \sqrt{\log T} \,.
\end{equation*}
The proof is concluded by decomposing according to the value inside the max term.

\section{Useful Results}

In this section, for the sake of completeness, we provide previous results used in our analysis together with a small lemma.

\begin{lemma}[Corollary 5.5 of \cite{lattimore2018bandit}]
Let $X_1, \ldots, X_n$ be $n$ independent $\sigma^2$-subgaussian random variables. Then for any $\varepsilon \geq 0$, it holds that
\begin{equation*}
    \mathbb{P}\left( \overline{X} \geq \varepsilon \right) \leq \exp\left( - \frac{n \varepsilon^2}{2 \sigma^2} \right)
\end{equation*}
where $\overline{X} = \frac{1}{n} \sum_{i=1}^n X_i$.
\label{lemma:concentration_bound_mean_subgauss}
\end{lemma}

\begin{lemma}[Lemma 2 of \cite{bayati2020optimal}]
Let $Q$ be a distribution with mean $\mu$ such that $Q - \mu$ is 1-subgaussian. Let $\{X_i\}_{i=1}^n$ be i.i.d.\ samples from distribution $Q$, $S_n = \sum_{i=1}^n X_i$ and $M_n = S_n / n$. Then for any $\delta > 0$, we have
\begin{equation*}
    \mathbb{P}\left( \exists n : M_n < \mu - \delta \right) \leq \exp\left(- \delta^2 / 2\right) \,.
\end{equation*}
\label{lemma:concentration_bound}
\end{lemma}

\begin{lemma}[Hoeffding's inequality]
Let $X_1, \ldots, X_n$ be independent bounded random variables supported in $[0, 1]$. For all $t \geq 0$, we have
\begin{equation*}
    \mathbb{P}\left( \frac{1}{n} \sum_{i=1}^n \left( X_i - \mathbb{E}[X_i] \right) \geq t \right) \leq \exp\left( - 2 n t^2 \right)
\end{equation*}
and 
\begin{equation*}
    \mathbb{P}\left( \frac{1}{n} \sum_{i=1}^n \left( X_i - \mathbb{E}[X_i] \right) \leq - t \right) \leq \exp\left( - 2 n t^2 \right) \,.
\end{equation*}
\label{lemma:hoeffding}
\end{lemma}

\begin{lemma}
Let a and b be two real numbers. Then the following holds true
\begin{equation*}
    \floor*{a+b} - \ceil*{a-b} \geq \floor*{2b} - 1 \,.
\end{equation*}
\label{lemma:1}
\end{lemma}

\begin{proof}
We have
\begin{align*}
    \floor*{a+b} - \ceil*{a-b} &= \floor*{a+b} + \floor*{b-a} \\
    &\geq \floor*{a+b + b - a} - 1\\
    &= \floor*{2b} - 1 \\
\end{align*}
where we used respectively that, $\ceil*{x} = - \floor*{-x}$ and $\floor*{x+y} \leq \floor*{x} + \floor*{y} + 1$.
\end{proof}

\section{Further experiments}

In this section, we evaluate the standard \Greedy{} algorithm, that considers all arms, in several bandit models to once again highlight its competitive performance in some cases compared to the state-of-the-art.

\subsection{Linear bandits}
\label{app:exp_linear}

In the linear bandit model, for each round $t$, the learner is given the decision set $\mathcal{A}_t \subset \mathbb{R}^d$, from which she chooses an action $A_t \in \mathcal{A}_t$ and receives reward $X_t=\langle \theta_\star, A_t\rangle + \eta_t$, where $\theta_\star \in \mathbb{R}^d$ is an unknown parameter vector and $\eta_t$ is some i.i.d.\ white noise, usually assume 1-subgaussian.
In this model, the \Greedy{} algorithm consists of two phases: firstly, it computes the regularized least-square estimator of $\theta$; then, it plays the arm in the action set that maximizes the linear product with the estimator of $\theta$.

Here we consider a problem with a large dimension relatively to the time horizon. Precisely, we fix $d=50$, a time horizon $T=2500$ and the noise is a standard Gaussian distribution. The set of arms consists of the unit ball and the parameter $\theta$ is randomly generated on the unit sphere.
We compare \Greedy{} with \textsc{LinUCB} \citep{abbasi2011improved} and \textsc{BallExplore} \citep{deshpande2012linear}, an algorithm specifically designed for such a setting. The regularization term $\lambda$ is set at 1 for \Greedy{} and \textsc{LinUCB}, the confidence term $\delta=\frac{1}{T}$ for \textsc{LinUCB} and the parameter $\Delta=d$ for BallExplore.
Results, displayed on Figure \ref{fig:exp_linear_unite_ball}, are averaged over 50 iterations. Shaded area represents 2 times the standard deviation for each algorithm.

\begin{figure}
    \centering
    \includegraphics[width=0.5\linewidth]{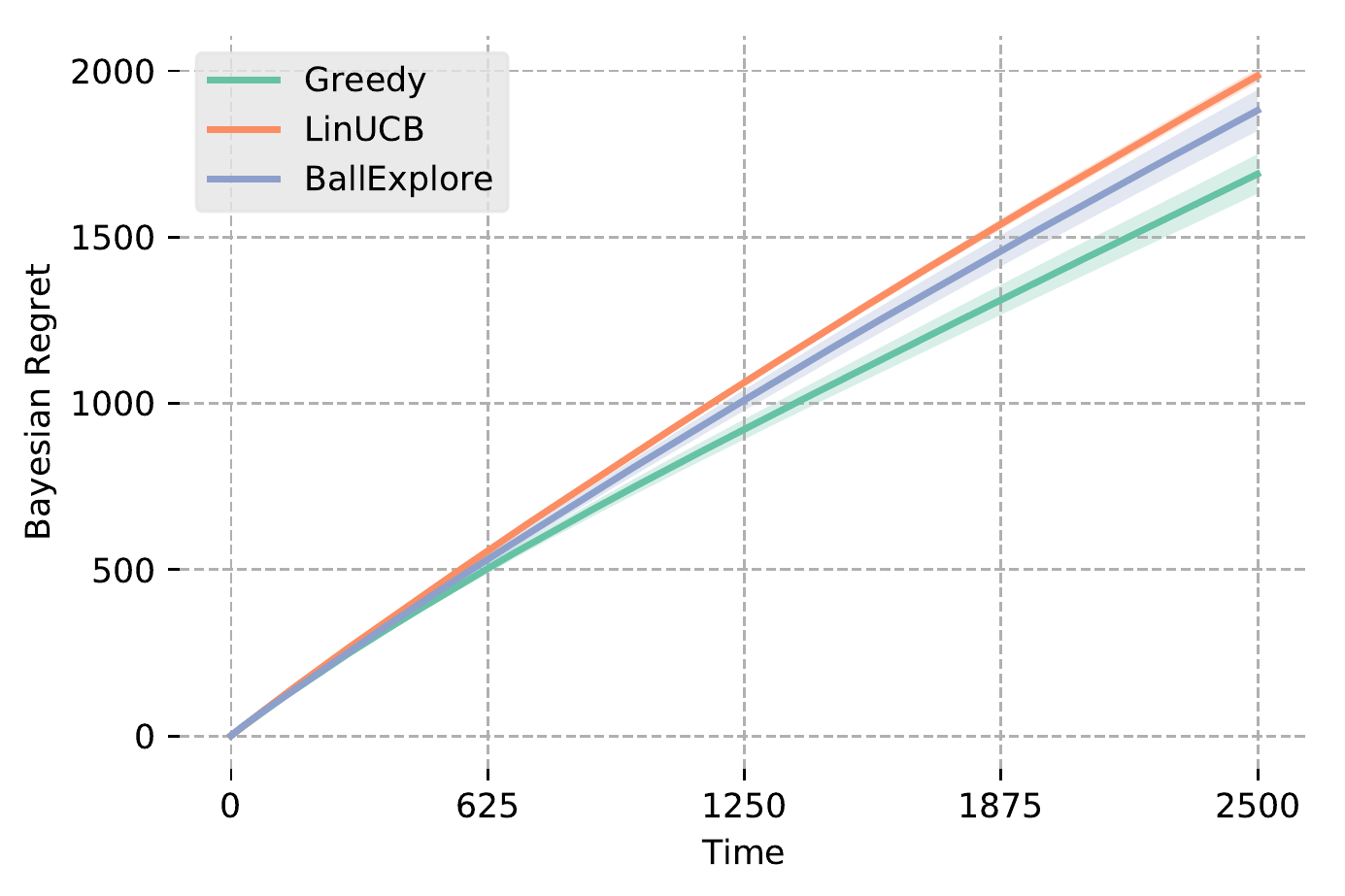}
    \caption{Bayesian regret of various algorithms as a function of time in a linear bandit problem.}
    \label{fig:exp_linear_unite_ball}
\end{figure}

We see that \Greedy{} outperforms both \textsc{LinUCB} and \textsc{BallExplore}; in particular the regret of \Greedy{} is sublinear.
Another point that we have not emphasized so far is the computational complexity. Until now, the difference in terms of computation was rather insignificant. This is no longer the case for algorithms designed for linear bandits as they must solve an optimization problem at each round. For example, in this simulation, the iteration time on a single-core processor is 70 seconds for \Greedy{}, 678 sec.\ for \textsc{LinUCB} and 1031 sec.\ for \textsc{BallExplore}. In words, \Greedy{} is roughly ten times faster than \textsc{LinUCB} and fifteen times faster than \textsc{BallExplore}.

\subsection{Cascading bandits}

We now consider a special, but popular, case of stochastic combinatorial optimization under semi-bandit feedback called the cascading bandit problem. 
Formally, there are $L \in \mathbb{N}$ ground items and at each round $t$, the agent recommends a list $\left( a_1^t, \dots, a_K^t \right)$ of $K \leq L$ items to the user. The user examines the list, from the first item to the last, and clicks on the first attractive item, if any. A weight $w(l) \in [0,1]$ is associated to each  item $l \in [L]$, which denotes the click probability of the item. The reward of the agent at round $t$ is given by $1 - \prod_{K=1}^K \left( 1 - w(a_k^t) \right) \in \{ 0, 1\}$ and she receives feedback for each $k\in [K]$ such that $k \leq c_t = \min\left\{1 \leq k \leq K : w_t(a_k^t) = 1\right\}$ where $w_t(a_k^t) \sim \text{Bernoulli}(w(a_k^t))$ and we assume that the minimum over an empty set is $\infty$.
In this setting, the \Greedy{} algorithm outputs a list consisting of the $K$ best empirical arms. 
The goal of these experiments is to study in which regimes, as a function of $L$ and $K$, the \Greedy{} algorithm might be preferable to the state-of-the-art.

We reproduce the experiments of \citet{kveton2015cascading} in the Bayesian setting.
We compare \Greedy{} with \textsc{CascadeKL-UCB} \citep{kveton2015cascading} and \textsc{TS-Cascade} \citep{cheung2019thompson}. \Greedy{} and \textsc{CascadeKL-UCB} share the same initialization which is to select each item once as the first item on the list. For each algorithm, the list is ordered from the largest index to the smallest one. We consider two scenarios: on the first one, the prior on the mean rewards is a uniform distribution over $[0, 1]$ while on the second scenario, we consider a more realistic Beta(1, 3) distribution so that most arms have low mean rewards. The time horizon is set at $T=10000$. The regret and standard deviation of each algorithm, averaged over 100 iterations, are reported in Table \ref{tab:cascade_unif} and \ref{tab:cascade_beta} for different values of $L$ and $K$.

\begin{table}
  \caption{Bayesian regret of various algorithms in cascading bandit problems with a uniform prior.}
  \label{tab:cascade_unif}
  \centering
  \begin{tabular}{cc|ccc}
    \toprule
    L & K & \Greedy{} & \textsc{CascadeKL-UCB} & \textsc{TS-Cascade} \\
    \midrule
    16 & 2 & 176.1 $\pm$ 26.4 & \textbf{48.1 $\pm$ 2.7} & 109.7 $\pm$ 1.8 \\
    16 & 4 & 10.2 $\pm$ 1.9 & \textbf{9.9 $\pm$ 1.0} & 28.4 $\pm$ 0.9 \\
    16 & 8 & \textbf{0.7 $\pm$ 0.2} & \textbf{0.7 $\pm$ 0.1} & 3.6 $\pm$ 0.3 \\
    32 & 2 & 166.1 $\pm$ 22.8 & \textbf{58.7 $\pm$ 3.5} & 178.7 $\pm$ 2.5 \\
    32 & 4 & \textbf{6.7 $\pm$ 0.9} & 10.1 $\pm$ 0.8 & 47.0 $\pm$ 1.0 \\
    32 & 8 & \textbf{0.2 $\pm$ 0.03} & 0.7 $\pm$ 0.08 & 8.3 $\pm$ 0.4 \\
    64 & 2 & 135.5 $\pm$ 15.6 & \textbf{76.6 $\pm$ 3.7} & 288.6 $\pm$ 2.6 \\
    64 & 4 & \textbf{6.5 $\pm$ 0.5} & 12.5 $\pm$ 0.6 & 80.3 $\pm$ 1.3 \\
    64 & 8 & \textbf{0.3 $\pm$ 0.02} & 0.9 $\pm$ 0.07 & 16.6 $\pm$ 0.5 \\
    128 & 2 & 133.1 $\pm$ 12.4 & \textbf{107.4 $\pm$ 4.8} & 442.6 $\pm$ 3.4 \\
    128 & 4 & \textbf{9.4 $\pm$ 0.3} & 18.0 $\pm$ 0.8 & 127.4 $\pm$ 1.5 \\
    128 & 8 & \textbf{0.5 $\pm$ 0.02} & 1.5 $\pm$ 0.1 & 27.9 $\pm$ 0.6 \\
    256 & 2 & \textbf{137.2 $\pm$ 10.6} & 151.0 $\pm$ 5.6 & 605.7 $\pm$ 3.1 \\
    256 & 4 & \textbf{16.6 $\pm$ 0.2} & 26.9 $\pm$ 1.0 & 179.5 $\pm$ 1.4 \\
    
    256 & 8 & \textbf{1.0 $\pm$ 0.03} & 1.8 $\pm$ 0.1 & 39.9 $\pm$ 0.5 \\
    \bottomrule
  \end{tabular}
\end{table}

\begin{table}
  \caption{Bayesian regret of various algorithms in cascading bandit problems with a Beta(1, 3) prior.}
  \label{tab:cascade_beta}
  \centering
  \begin{tabular}{cc|ccc}
    \toprule
    L & K & \Greedy{} & \textsc{CascadeKL-UCB} & \textsc{TS-Cascade} \\
    \midrule
    16 & 2 & 590.4 $\pm$ 83.5 & 207.9 $\pm$ 5.2 & \textbf{199.5 $\pm$ 3.6} \\
    16 & 4 & 304.8 $\pm$ 35.7 & 116.4 $\pm$ 4.2 & \textbf{103.2 $\pm$ 2.9} \\
    16 & 8 & 97.9 $\pm$ 11.7 & 39.6 $\pm$ 2.1 & \textbf{34.4 $\pm$ 1.6} \\
    32 & 2 & 433.1 $\pm$ 49.1 & \textbf{330.7 $\pm$ 8.3} & 333.7 $\pm$ 3.8 \\
    32 & 4 & 192.2 $\pm$ 23.1 & 166.2 $\pm$ 6.0 & \textbf{163.3 $\pm$ 3.7} \\
    32 & 8 & \textbf{38.7 $\pm$ 5.3} & 50.1 $\pm$ 2.9 & 54.6 $\pm$ 1.9 \\
    64 & 2 & 576.2 $\pm$ 55.8 & \textbf{485.8 $\pm$ 11.2} & 540.1 $\pm$ 4.8 \\
    64 & 4 & \textbf{144.2 $\pm$ 12.3} & 207.5 $\pm$ 6.8 & 246.1 $\pm$ 4.1 \\
    64 & 8 & \textbf{20.3 $\pm$ 1.8} & 49.2 $\pm$ 2.2 & 76.4 $\pm$ 1.6 \\
    128 & 2 & \textbf{575.2 $\pm$ 40.1} & 710.9 $\pm$ 16.3 & 843.4 $\pm$ 4.7 \\
    128 & 4 & \textbf{100.8 $\pm$ 5.5} & 270.6 $\pm$ 7.4 & 372.9 $\pm$ 3.7 \\
    128 & 8 & \textbf{18.0 $\pm$ 0.6} & 60.7 $\pm$ 2.0 & 115.7 $\pm$ 1.4 \\
    256 & 2 & \textbf{522.5 $\pm$ 32.4} & 1068.3 $\pm$ 26.1 & 1235.1 $\pm$ 6.3 \\
    256 & 4 & \textbf{125.1 $\pm$ 3.8} & 380.0 $\pm$ 10.3 & 551.1 $\pm$ 3.85 \\
    256 & 8 & \textbf{27.3 $\pm$ 0.4} & 86.4 $\pm$ 2.6 & 174.8 $\pm$ 1.5 \\
    \bottomrule
  \end{tabular}
\end{table}

As expected by the Bayesian setting, \Greedy{} outplays the state-of-the-art when the number of arms $L$ is large. 
Even more interesting is that, as the number of recommended items $K$ gets larger the regret of \Greedy{} decreases at a faster rate than the other algorithms. Our intuition is that the conservatism of standard bandit algorithms is amplified as $K$ increases and this is further exacerbated by the cascade model where items at the bottom of the list may not get a feedback. On the contrary, the \Greedy{} algorithm quickly converges to a solution that uniquely depends on past individual performances of arms. 
In addition, the contrast between the performance of \Greedy{} and the state-of-the-art is even more striking in the second scenario. This is not particularly surprising as the Beta(1, 3) distribution gives rise to harder problems for the considered time horizon.

\subsection{Mortal bandits}

We now consider the mortal bandit problem where arms die and new ones appear regularly (in particular, an arm is not always available contrary to the standard model). In this setting, the \Greedy{} algorithm pulls the best empirical arm available. As previous work considered a large number of arms, state-of-the-art algorithms in this setting, e.g. \textsc{AdaptiveGreedy} \citep{chakrabarti2009mortal}, emphasis an hidden subsampling of arms due to their initialization. They further required a careful (manual) tuning of their parameter for optimal performance. Consequently, we compare \Greedy{} to a standard bandit algorithm extended to this model and we consider a small number of arms. Similarly to the last setting, the goal is to observe in which regimes, as a function of the mean lifetime of arms, \Greedy{} might be preferable.

We repeat the experiments of \citet{chakrabarti2009mortal} with $K=100$ arms. The number of arms remains fixed throughout the time horizon $T$, that is when an arm dies, it is immediately replaced by another one. The time horizon $T$ is set at 10 times the mean lifetime of the arms. The lifetime of arm $k$, denoted $L_k$, is drawn i.i.d.\ from a geometric distribution with mean lifetime $L$; this arm dies after being available for $L_k$ rounds. 
We consider logarithmically spaced values of mean lifetimes.
We also assume that arms are Bernoulli random variables. We consider two scenarios: in the first one, mean rewards of arms are drawn i.i.d.\ from a uniform distribution over [0, 1], while in the second scenario they are drawn from a Beta(1, 3) distribution. We compare the \Greedy{} algorithm with \textsc{Thompson Sampling} \citep{agrawal2012analysis}. Results are averaged over 100 iterations and are reported on Figure \ref{fig:exp_mortal}. Shaded area represents 0.5 standard deviation for each algorithm.

\begin{figure}
  \floatconts
  {fig:exp_mortal}
  {\caption{Bayesian regret of various algorithms as a function of the expected lifetime of arms in mortal bandit problems.}}
  {%
    \subfigure[Uniform prior]{
      \includegraphics[width=0.49\linewidth]{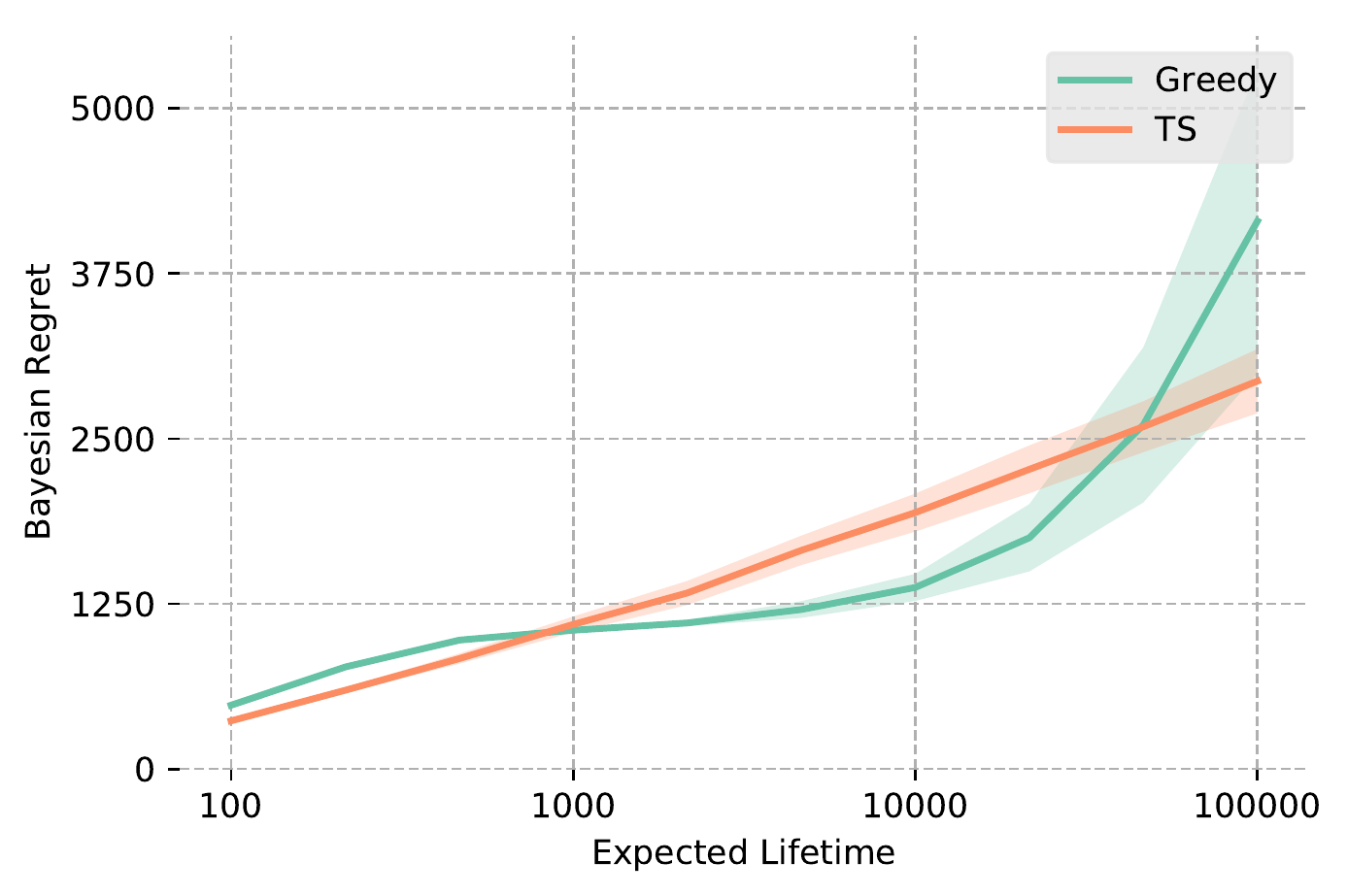}}%
    \subfigure[Beta(1, 3) prior]{
      \includegraphics[width=0.49\linewidth]{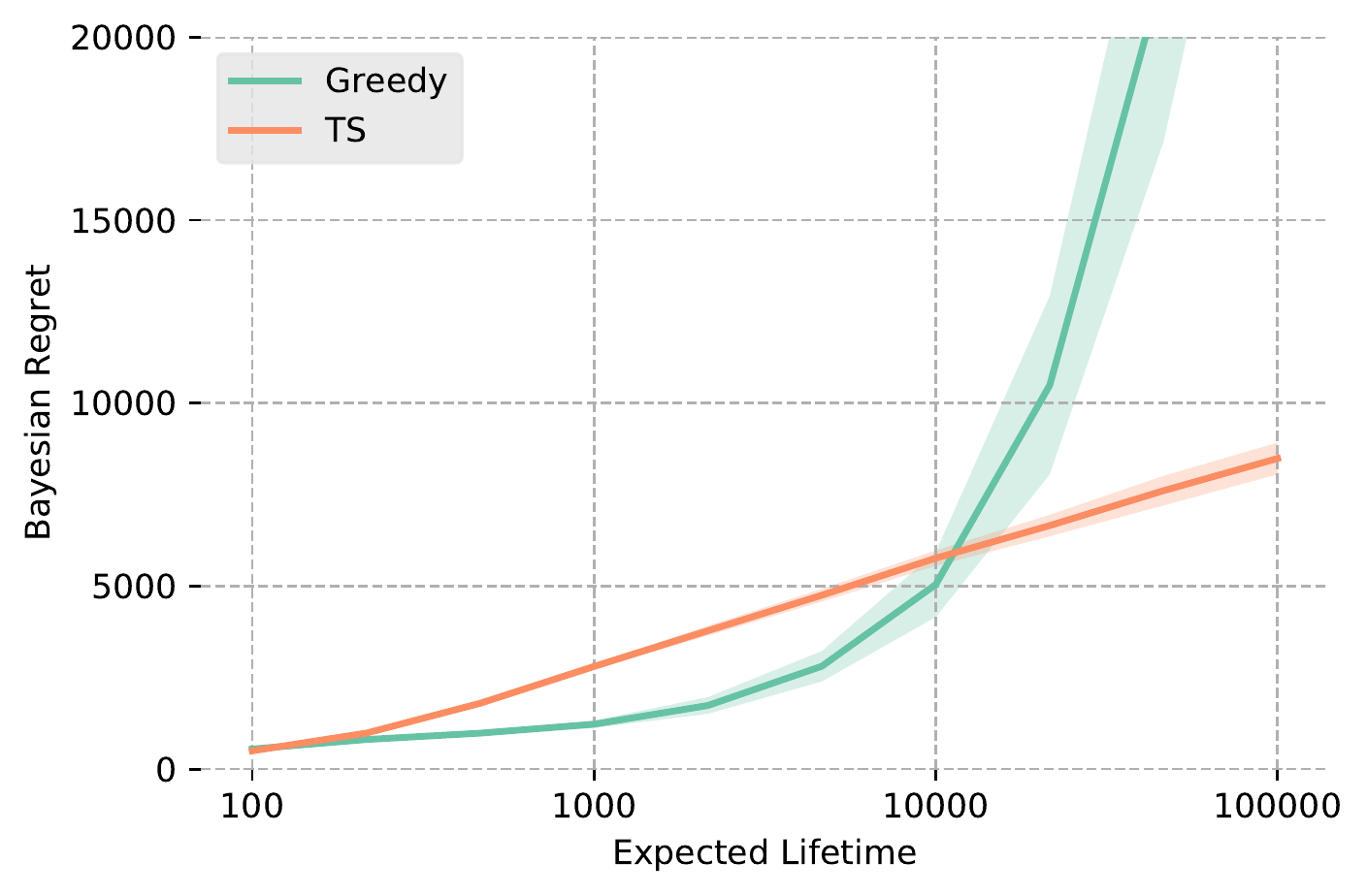}}
  }
\end{figure}

As expected, \Greedy{} outperforms \textsc{Thompson Sampling} for intermediate expected lifetime and vice versa for long lifetime. And for short lifetime, as we previously saw, a sub-sampling of arms could have considerably improve the performance of both algorithms.  

\subsection{Budgeted bandits}

We now consider the budgeted bandit problem. In this model, the pull of arm $k$ at round $t$ entails a random cost $c_k(t)$. Moreover, the learner has a budget $B$, which is a known parameter, that will constrain the total number of pulls.
In this setting, the index of an arm in the \Greedy{} algorithm is the average reward divided by the average cost.
Like before, the objective is to evaluate in which regimes with respect to the budget $B$, \Greedy{} might be preferable to a state-of-the-art algorithm.

We reproduce the experiments of \citet{xia2016budgeted}. Specifically, we study two scenarios with $K=100$ arms in each. The first scenario considers discrete costs; both the reward and the cost are sampled from Bernoulli distributions with parameters randomly sampled from $(0, 1)$. The second scenario considers continuous costs; the reward and cost of an arm is sampled from two different Beta distributions, the two parameters of each distribution are uniformly sampled from $[1,5]$. The budget is chosen from the set $\{100, 500, 1000,5000,10000 \}$.
We compare \Greedy{} to \textsc{Budget-UCB} \citep{xia2016budgeted} and \textsc{BTS} \citep{xia2015thompson}. 
The results of simulations are displayed in Figure \ref{fig:exp_budgeted} and are averaged over 500 runs. Shaded area represents 0.5 standard deviation for each algorithm.

\begin{figure}
  \floatconts
  {fig:exp_budgeted}
  {\caption{Regret of various algorithms as a function of the budget in budgeted bandit problems.}}
  {%
    \subfigure[Discrete costs]{
      \includegraphics[width=0.49\linewidth]{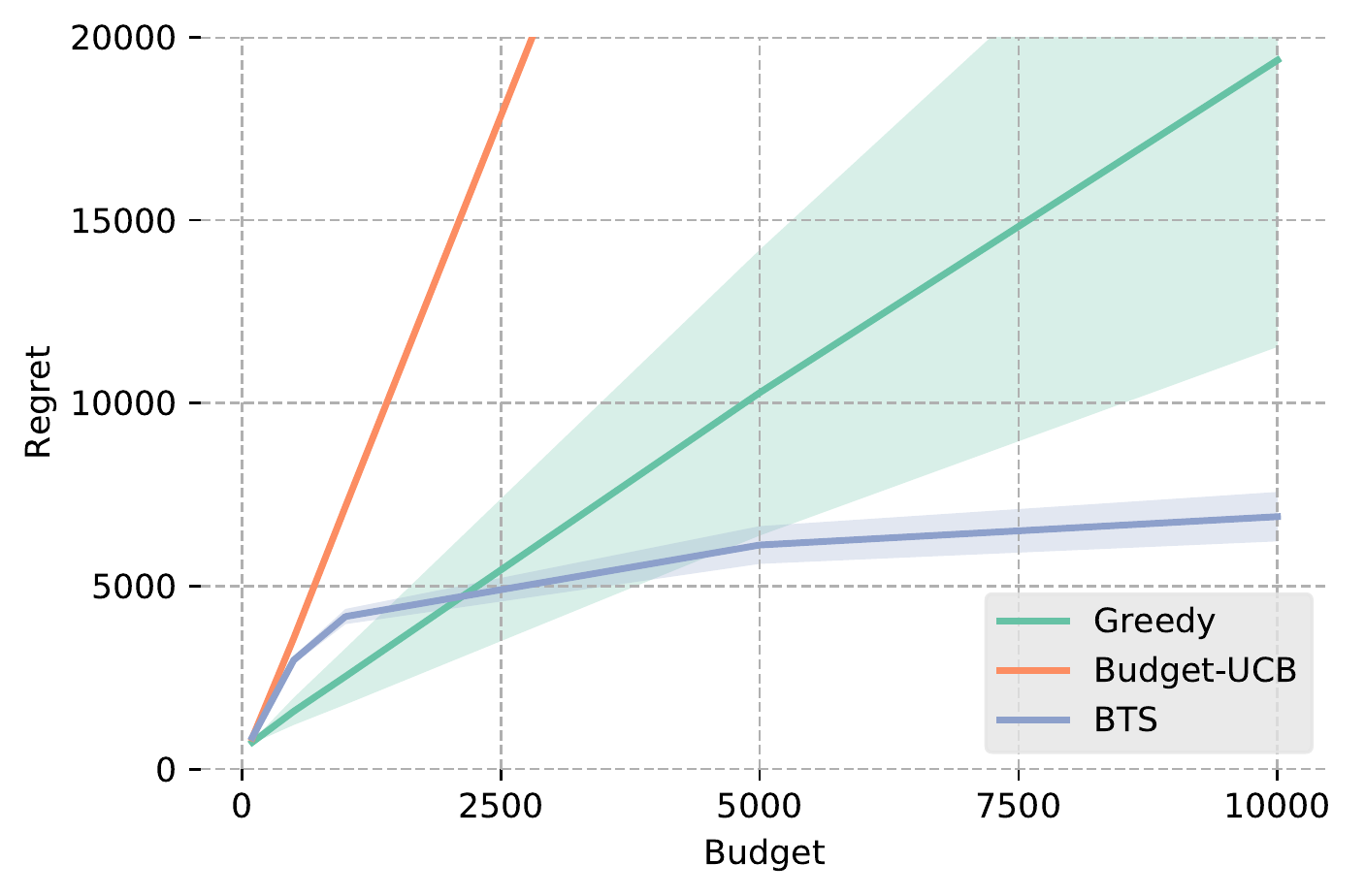}}%
    \subfigure[Continuous costs]{
      \includegraphics[width=0.49\linewidth]{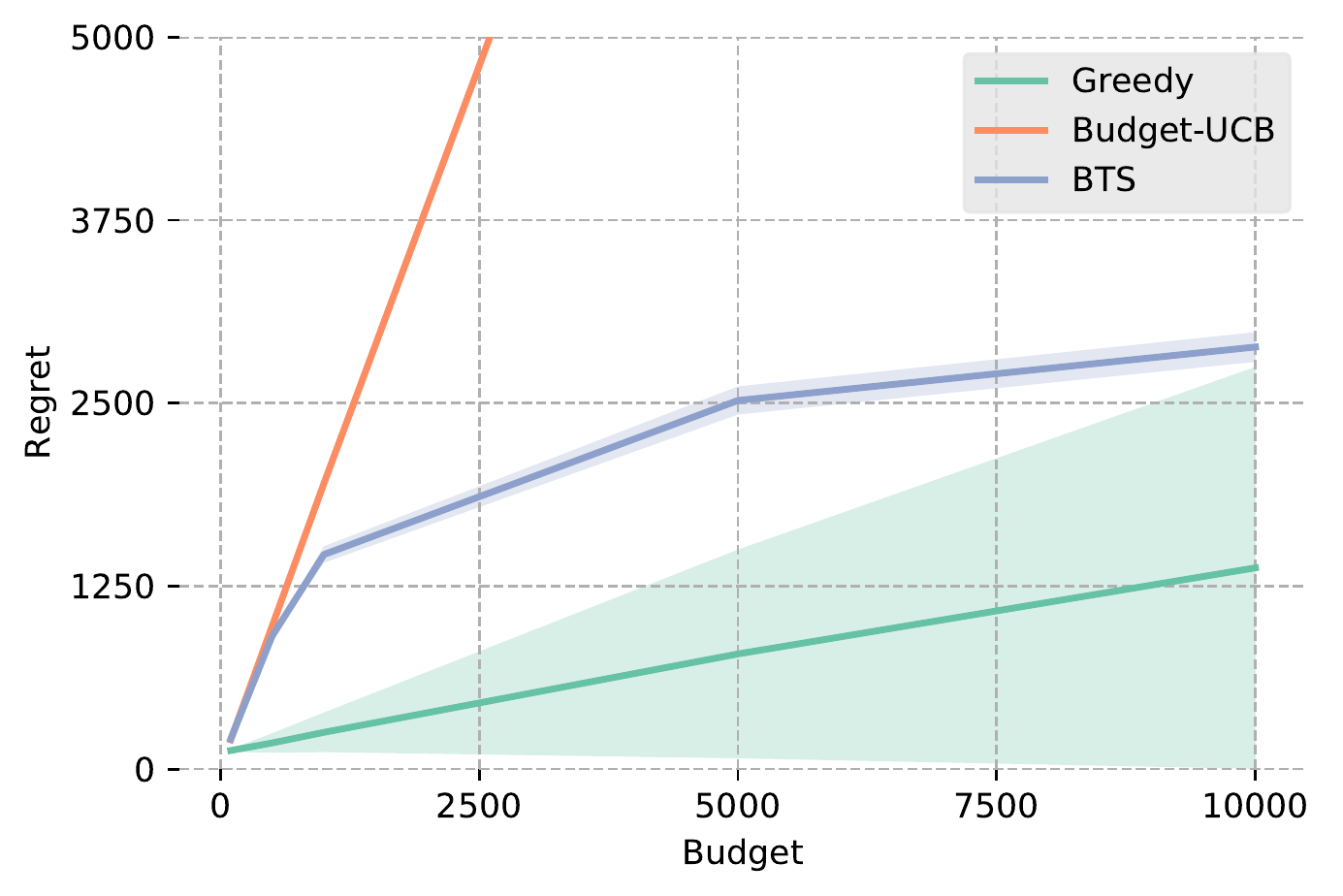}}
  }
\end{figure}

Interestingly, in this setting the interval of budgets for which \Greedy{} outperforms baseline algorithms is extremely small for discrete costs and large for continuous costs. In the latter case, even for large budget \Greedy{} has a lower expected regret than \textsc{BTS}. Nonetheless it suffers from a huge variance which makes its use risky in practice.

\end{document}